%% file: FedCluster_bigdata2020.tex
\documentclass[conference]{IEEEtran}
\IEEEoverridecommandlockouts
\usepackage{cite}
\usepackage{amsmath,amssymb,amsfonts}
\usepackage{algorithmic}
\usepackage{graphicx}
\usepackage{textcomp}
\usepackage{xcolor}
\def\BibTeX{{\rm B\kern-.05em{\sc i\kern-.025em b}\kern-.08em
    T\kern-.1667em\lower.7ex\hbox{E}\kern-.125emX}}

\usepackage[utf8]{inputenc} 
\usepackage[T1]{fontenc}    
\usepackage{hyperref}       
\usepackage{url}            
\usepackage{booktabs}       
\usepackage{amstext,amsthm,mathrsfs}       
\usepackage{nicefrac}       
\usepackage{microtype}      
\usepackage{enumitem}
\usepackage{epstopdf}
\usepackage{cleveref}
\usepackage{algorithm}
\usepackage[algo2e]{algorithm2e}
\usepackage[normalem]{ulem}

\newtheorem{lemma}{Lemma}
\newtheorem{thm}{Theorem}

\newtheorem{assum}{Assumption}

\allowdisplaybreaks  

\begin{document}

\title{FedCluster: Boosting the Convergence of Federated Learning via Cluster-Cycling\\
\thanks{*These authors contributed equally to this work.}
}

\author{\IEEEauthorblockN{Cheng Chen\textsuperscript{*}}
\IEEEauthorblockA{Department of ECE \\
{University of Utah}\\
u0952128@utah.edu}
\and
\IEEEauthorblockN{Ziyi Chen\textsuperscript{*}}
\IEEEauthorblockA{Department of ECE \\
{University of Utah}\\
ziyi.chen@utah.edu}
\and
\IEEEauthorblockN{Yi Zhou}
\IEEEauthorblockA{Department of ECE \\
{University of Utah}\\
yi.zhou@utah.edu}
\and
\IEEEauthorblockN{Bhavya Kailkhura}
\IEEEauthorblockA{{Lawrence Livermore National Lab} \\
San Francisco, US \\
kailkhura1@llnl.gov}
}

\maketitle

\begin{abstract}
	We develop FedCluster -- a novel federated learning framework with improved optimization efficiency, and investigate its theoretical convergence properties. The FedCluster groups the devices into multiple clusters that perform federated learning cyclically in each learning round. Therefore, each learning round of FedCluster consists of multiple cycles of meta-update that boost the overall convergence. In nonconvex optimization, we show that FedCluster with the devices implementing the local {stochastic gradient descent (SGD)} algorithm achieves a faster convergence rate than the conventional {federated averaging (FedAvg)} algorithm in the presence of device-level data heterogeneity. We conduct experiments on deep learning applications and demonstrate that FedCluster converges significantly faster than the conventional federated learning under diverse levels of device-level data heterogeneity for a variety of local optimizers. 
\end{abstract}


\input{introduction}

\input{fedcluster}

\input{analysis}

\input{experiments}

\section{Conclusion}
In this paper, we propose a novel federated learning framework named FedCluster. The FedCluster framework groups the edge devices into multiple clusters and activates them cyclically to perform federated learning in each learning round. We provide theoretical convergence analysis to show that FedCluster with local SGD achieves faster convergence than the conventional FedAvg algorithm in nonconvex optimization, and the convergence of FedCluster is less dependent on the device-level data heterogeneity. Our experiments on deep federated learning corroborate the theoretical findings. In the future, we expect and hope that FedCluster can be implemented in practical federated learning systems to demonstrate its fast convergence and provide great flexibility in scheduling the workload for the devices. Also, it is interesting to explore the impact of the clustering approach and the random cyclic order on the performance of FedCluster.




\appendices
\input{supp_exp}

\end{document}

%% file: introduction.tex
\section{Introduction}
\vspace{-5pt}
Federated learning has become an emerging distributed machine learning framework that enables edge computing at a large scale \cite{konevcny2016federated,li2020federated,mcmahan2016communication}. As opposed to traditional centralized machine learning that collects all the data at a central server to perform learning, federated learning exploits the distributed computation power and data of a massive number of edge devices to perform distributed machine learning while preserving full data privacy. In particular, federated learning has been successfully applied to the areas of Internet of things (IoT), autonomous driving, health care, etc. \cite{li2020federated}

The original federated learning framework was proposed in \cite{mcmahan2016communication}, where the federated averaging (FedAvg) training algorithm was developed. Specifically, in each learning round of FedAvg, a subset of devices are activated to  download a model from the cloud server and train the model using {their} local data for multiple {stochastic gradient descent (SGD)} iterations. Then, the devices upload the trained local models to the cloud, where the local models are aggregated and averaged to obtain an updated global model to be used in the next round of learning. In this learning process, the data are privately kept in the devices. However, the convergence rate of the FedAvg algorithm is heavily affected by the device-level data heterogeneity of the devices, which has been shown both empirically and theoretically to slow down the convergence of FedAvg \cite{li2019convergence,zhao2018federated}.

To alleviate the negative effect of device-level data heterogeneity and facilitate the convergence of federated learning, various algorithms have been developed in the literature. For example, the {federated proximal (FedProx)} algorithm studied in \cite{li2018federated, sahu2018convergence} proposed to regularize the local loss function with the square distance between the local model and the global model, which helps to reduce the heterogeneity of the local models. Other federated learning algorithms apply variance reduction techniques to reduce the variance of local stochastic gradients caused by device-level data heterogeneity, e.g., FedMAX \cite{chen2020fedmax}, FEDL \cite{dinh2019federated}, VRL-SGD \cite{liang2019variance}, FedSVRG \cite{konevcny2016federated, nagar2019privacy} and PR-SPIDER \cite{sharma2019parallel}. {Moreover, \cite{karimireddy2019scaffold} uses control variates to control the local model difference, which is similar to the variance reduction techniques.} {The MIME framework \cite{karimireddy2020mime} combines arbitrary local update algorithms with server-level statistics to control local model difference. In FedNova \cite{wang2020tackling}, the devices adopt different numbers of local SGD iterations and use normalized local gradients to adapt to the device-level data heterogeneity.} Although these federated learning algorithms achieve improved performance over the traditional FedAvg algorithm, their optimization efficiency is unsatisfactory due to the infrequent model update in the cloud, i.e., every round of these federated learning algorithms only performs one global update (i.e., the model average step) to update the global model. Such an update rule does not fully utilize the big data to make the maximum optimization progress, and the convergence rates of these federated learning algorithms do not benefit from the large population of the devices. Therefore, it is much desired to develop an advanced federated learning framework that achieves improved optimization convergence and efficiency with the same amount of resources as those of the existing federated learning frameworks.

In this paper, we propose a novel federated learning framework named FedCluster, in which the devices are grouped into multiple clusters to perform federated learning in a cyclic way. The proposed federated learning framework has the flexibility to implement any conventional federated learning algorithms and can boost their convergence with the same amount of computation \& communication resources. We summarize our contributions as follows.

\newpage

\subsection{Our Contributions}
We propose FedCluster, a novel federated learning framework that groups the devices into multiple clusters for flexible management of the devices. In each learning round, the clusters of devices are activated to perform federated learning in a cyclic order. In particular, the FedCluster framework has the following amenable properties: 1) The clusters of devices in FedCluster can implement any federated learning algorithm, e.g., FedAvg, FedProx, etc. Hence, FedCluster provides a general optimization framework for federated learning; and 2) In each learning round, FedCluster updates the global model multiple times while consuming the same amount of computation \& communication resources as those consumed by the traditional federated learning framework. 

Theoretically, {we proved that in nonconvex optimization,} FedCluster with the devices implementing the local SGD algorithm {converges at a sub-linear rate $\mathcal{O}(\frac{1}{\sqrt{TME}})$, where $T,M,E$ correspond to the number of learning rounds, clusters and local SGD iterations, respectively. As a comparison, the convergence rate of the conventional FedAvg is in the order of $\mathcal{O}(\frac{1}{\sqrt{TE}})$ \cite{li2019communication}, which is slower than FedCluster with local SGD by a factor of $\sqrt{M}$. In addition, our convergence rate only depends on cluster-level data heterogenity denoted as $H_{\text{cluster}}$ (see \eqref{H_cluster}), while FedAvg depends on a device-level data heterogeneity $H_{\textrm{device}}$ that is larger than $H_{\text{cluster}}$.} Therefore, FedCluster with local SGD achieves a faster convergence rate and suffers less from the device-level data heterogeneity than FedAvg.  


Empirically, we compare the performance of FedCluster with that of the conventional centralized federated learning with different local optimizers in deep learning applications. We show that FedCluster achieves significantly faster convergence than the conventional centralized federated learning under different levels of device-level data heterogeneity and local optimizers. We also explore the impact of the number of clusters and cluster-level data heterogeneity on the performance of FedCluster.

\subsection{Related Work}
The conventional federated learning framework along with the FedAvg algorithm was first introduced in \cite{mcmahan2016communication}. The convergence rate of FedAvg was studied in strongly convex optimization \cite{koloskova2020unified,li2019convergence,woodworth2020minibatch}, {convex optimization \cite{koloskova2020unified,khaled2019first,bayoumi2020tighter,woodworth2020minibatch}} and nonconvex optimization \cite{koloskova2020unified,li2019communication}. {\cite{koloskova2020unified} studied the convergence rate of decentralized SGD, a generalization of FedAvg.} To address the device-level data heterogeneity {issue} in federated learning, \cite{li2018federated, sahu2018convergence} proposed the FedProx algorithm that adds the square distance between the local model and the global model to the local loss function, \cite{reddi2020adaptive} applied adaptive optimizers such as Adagrad, Adam, and YOGI to the update of the global model, and variance reduction techniques have been applied to develop advanced federated learning algorithms including Federated SVRG \cite{konevcny2016federated, nagar2019privacy}, FEDL \cite{dinh2019federated}, FedMAX \cite{chen2020fedmax}, VRL-SGD \cite{liang2019variance} and PR-SPIDER \cite{sharma2019parallel}. \cite{yao2019federated} proposed a FedMeta algorithm that uses a keep-trace gradient descent strategy originated from model-agnostic meta-learning \cite{finn2017model}, but the local updates involve high-order derivatives. \cite{ramazanli2020adaptive} developed an ASD-SVRG algorithm to address the heterogeneity of the local smoothness parameters, but the algorithm suffers from a high computation cost due to frequent communication and full gradient computation in the variance reduction scheme. {\cite{sattler2019clustered} performed FedAvg and bisection clustering in turn until convergence, which is time consuming without theoretical convergence guarantee.} {\cite{castiglia2020multi} also used clustering scheme where each cluster performs FedAvg and periodically communicate with neighbor clusters in a decentralized manner.} The semi-federated learning framework \cite{chen2020semi} also proposed to cluster the devices, but the devices take only one local SGD update per learning round and require D2D communication. The Federated Augmentation (FAug) \cite{jeong2018communication} used generative adversarial networks to augment the local device data{set} to make it closer to an i.i.d. dataset. 

Some other works focus on personalized federated learning. For example, {\cite{eichner2019semi} trained different local models instead of a shared global model.} \cite{deng2020adaptive,hanzely2020federated} trained a mixture of global and local models and \cite{agarwal2020federated} combined the global and local models to make prediction for online learning. Readers can refer to \cite{kulkarni2020survey} for an overview of personalized federated learning. 
Other studies of federated learning include privacy \cite{hu2020cpfed,jin2020stochastic,wei2020framework}, adversarial attacks \cite{bhagoji2018analyzing,jin2020stochastic} and fairness \cite{mohri2019agnostic}. Please refer to \cite{kairouz2019advances} for a comprehensive overview of these topics.

%% file: fedcluster.tex
\section{FedCluster: Federated Learning with Cluster-Cycling}
In this section, we introduce the FedCluster framework and compare it with the traditional federated learning framework. To provide a fair comparison, we impose the following practical resource constraint on the devices in federated learning systems. 

\begin{assum}[Resource constraint]\label{assum: resource}
	In each round of federated learning, every device has a {maximum} budget of two communications (one download and one upload) and a fixed number of local updates. 
\end{assum} 

Fig. \ref{fig: 1} (Left) illustrates a particular learning round of the traditional centralized federated learning system. To elaborate, in each learning round,  a subset of devices are activated to download a model from the cloud server and train the model using {their} local data for multiple iterations. The training algorithm can be any stochastic optimization algorithm, e.g., local SGD in FedAvg \cite{mcmahan2016communication} and regularized local SGD in FedProx \cite{li2018federated}. After the local training, the activated devices upload their trained local models to the cloud, where the local models are aggregated and averaged to obtain an updated model to be used in the next learning round. We note that the devices of the traditional federated learning system satisfy the resource constraint in \Cref{assum: resource}. However, the convergence of the traditional federated learning process is slowed down by two important factors: 1) the heterogeneous local datasets of the devices and 2) the infrequent global model update in each learning round. 

\begin{figure*}[htbp]
	\vspace{-2mm}
	\begin{minipage}{.5\textwidth}
		\centering
	\includegraphics[width=.8\linewidth]{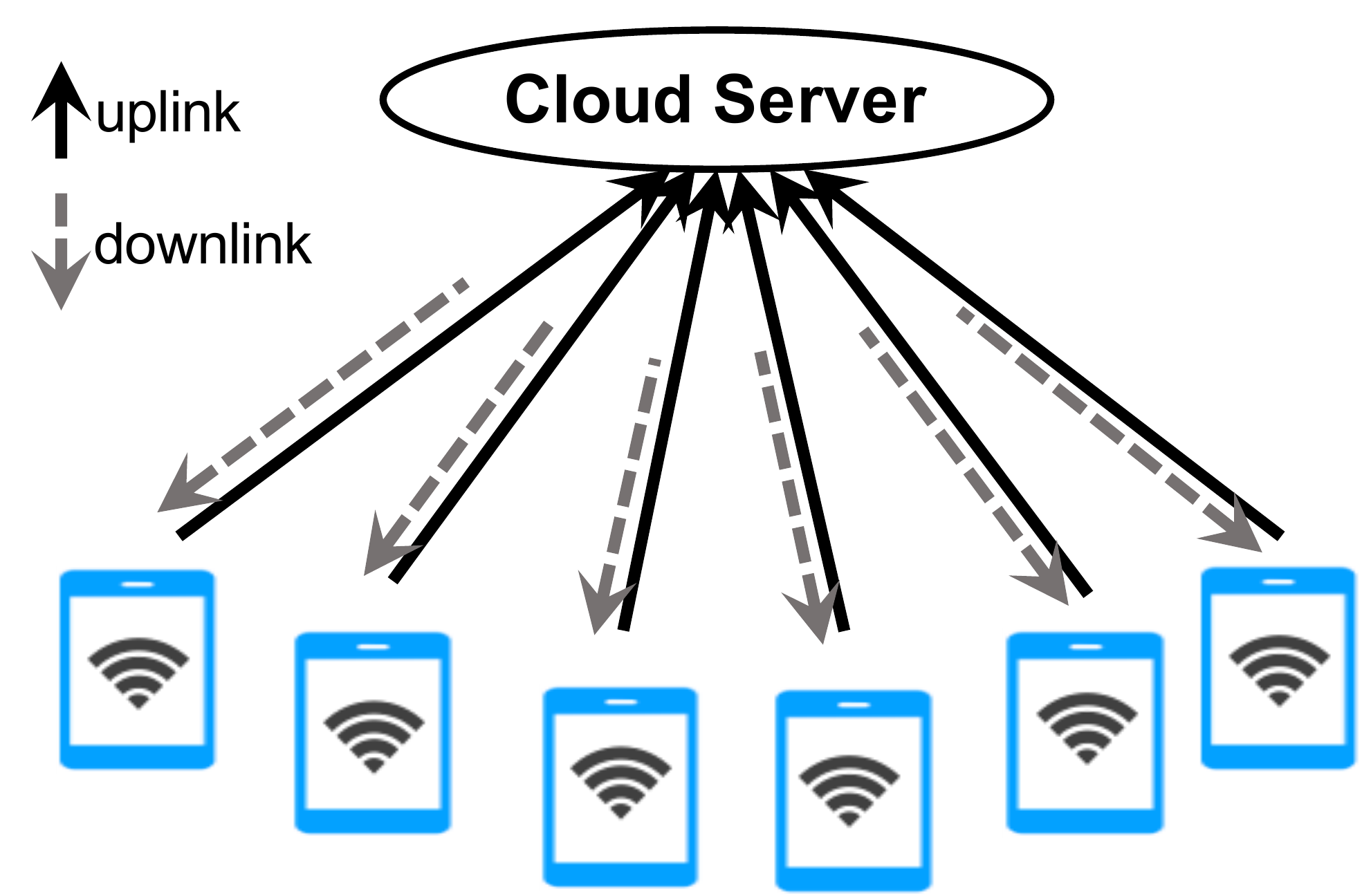}
	\end{minipage} 
	\begin{minipage}{.5\textwidth}
		\centering
	\includegraphics[width=.7\linewidth]{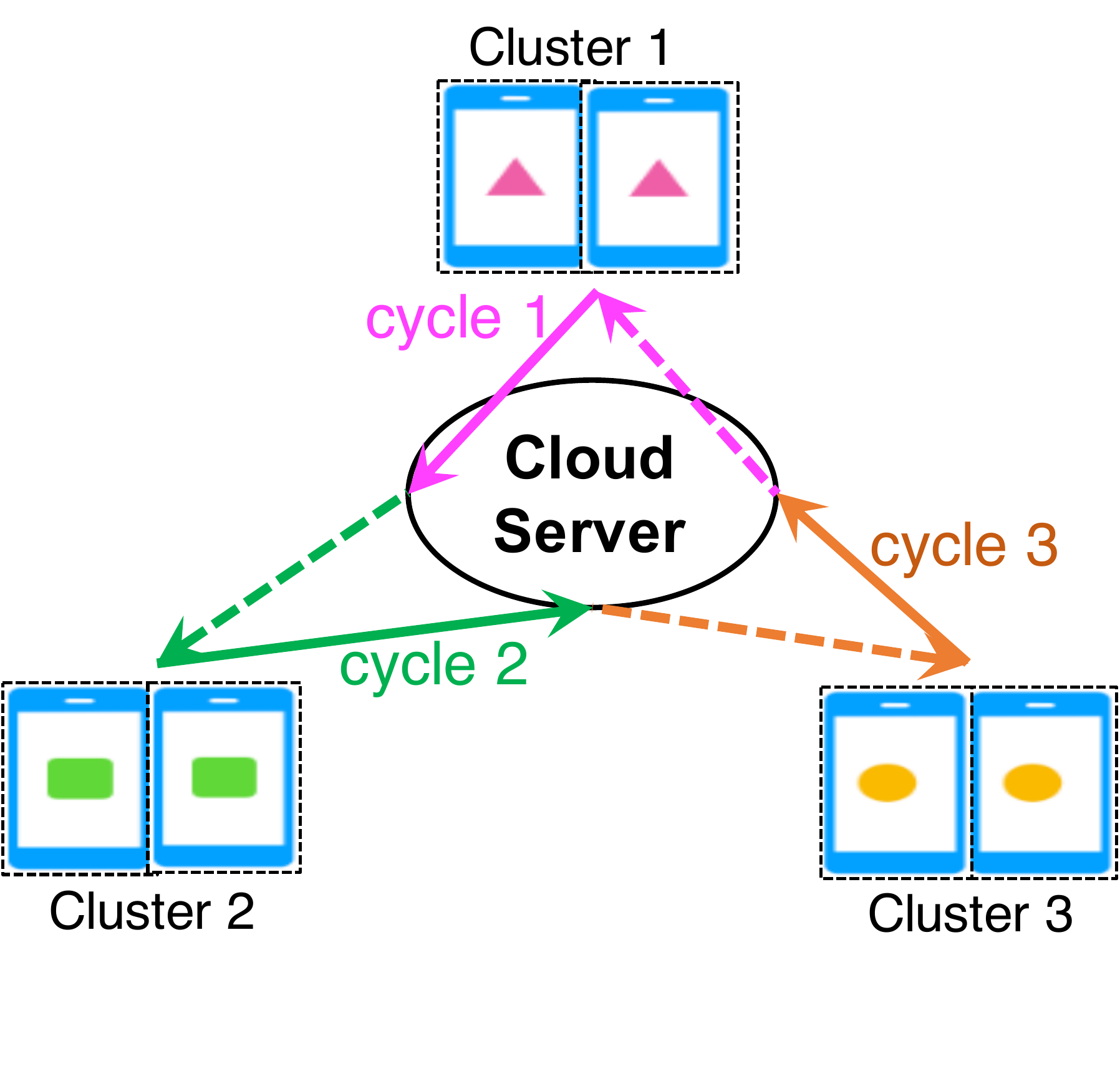}
	\end{minipage} 
	\vspace{-4mm}
	\caption{Left: Traditional federated learning system. Right: FedCluster system with cluster-cycling.}\label{fig: 1} 
	\vspace{-2mm}
\end{figure*}

To improve the optimization efficiency and flexibility of federated learning, we propose the FedCluster framework as illustrated in Fig. \ref{fig: 1} (Right). To elaborate, in the FedCluster system, devices are grouped into multiple clusters using a certain clustering approach (elaborated later). In each learning round, the system performs multiple cycles of federated learning through the clusters in a cyclical way. Specifically, as illustrated in Fig. \ref{fig: 1} (Right), in the first cycle of a learning round, a subset of devices of Cluster 1 are activated to perform federated learning, i.e., they download a model from the cloud server, perform local trainings using a certain algorithm (e.g., SGD) and upload the trained local models to the cloud for model averaging. Then, in the next cycle, a subset of devices of Cluster 2 are activated to perform federated learning. Following this strategy, all the clusters take turns to perform federated learning in a cyclic order. We note that the devices of FedCluster satisfy the resource constraint in \Cref{assum: resource}, as each device is only activated {at most} once per learning round. 

\textbf{Comparison:} The key differences between the FedCluster framework and the traditional federated learning framework are in two-fold. First, in each learning round, FedCluster updates the global model multiple times (equals to number of clusters), whereas the traditional federated learning updates it only once. Hence, FedCluster is expected to make more optimization progress per learning round. In fact, as the devices in federated learning are usually busy and unavailable, the clusters in FedCluster provide much flexibility to schedule the learning tasks for the devices and make frequent updates on the global model (see the discussion in the next paragraph.).
Second, the traditional federated learning process suffers from the device-level data heterogeneity. In comparison, as we show later in the analysis section, the convergence of the FedCluster learning process is affected by a smaller cluster-level data heterogeneity. 

\textbf{Generality and flexibility:} We further elaborate various aspects of generality and flexibility of the FedCluster framework.
\begin{itemize}[leftmargin=*,topsep=0pt, itemsep=0pt]
	\item {\em Generality:} The traditional federated learning framework in Fig. \ref{fig: 1} (Left) can be viewed as a special case of the FedCluster framework with only one cluster that includes all the devices. 	
	\item {\em Algorithm:} The clusters of FedCluster can implement any federated learning algorithms that are compatible with the traditional {federated learning} framework, e.g., FedAvg \cite{mcmahan2016communication}, FedProx \cite{li2018federated}, etc.	
	\item {\em Clustering: }  In FedCluster, the way to cluster the devices depends on the specific application scenario. Below we provide several representative clustering approaches. 
	\begin{enumerate}[leftmargin=*,topsep=0pt, itemsep=0pt]
		\item {Random uniform clustering:} The devices  are grouped into multiple clusters of equal size uniformly at random. In this case, the clusters are homogeneous and have similar data statistics. 
		\item {Timezone-based clustering:} In mobile networks where the devices are smart phones that are distributed over the world, one can group the devices based on either their timezones or GPS locations. This scenario fits the FedCluster system well because many smart phones are available only at a particular local time (e.g., midnight) to perform federated learning. In this case, the federated learning process cycles through the smart phones in different timezones.
		\item {Availability-based clustering:} A more general approach is to divide each learning round into multiple time slots. Each device determines its available time slot to perform federated learning. Then, the available devices within each time slot form a cluster.
	\end{enumerate} 
\end{itemize}


%% file: analysis.tex
\section{Convergence Analysis of FedCluster with Local SGD}
In this section, we analyze the convergence of FedCluster in smooth nonconvex optimization.

\subsection{Problem Formulation and Algorithm}
We consider a federated learning system that consists of $n$ devices. Each device $k$ possesses a local data set $\mathcal{D}^{(k)}$ with $|\mathcal{D}^{(k)}|$ data samples. Denote $\mathcal{D}= \cup_{k=1}^{n} \mathcal{D}^{(k)}$ as the total dataset and denote $p_k = |\mathcal{D}^{(k)}|/|\mathcal{D}|$ as the proportion of data possessed by the device $k$. Then, we aim to solve the following empirical risk minimization problem
\begin{align}
	\min_{w\in \mathbb{R}^d} ~ f(w) = \sum_{k=1}^n p_k f(w; \mathcal{D}^{(k)}), \tag{P}
\end{align}
where $f(w; \mathcal{D}^{(k)})=\frac{1}{|\mathcal{D}^{(k)}|} \sum_{\xi \in \mathcal{D}^{(k)}}f(w; \xi)$ corresponds to the average loss on the local dataset.

In FedCluster, we group all the devices $\{1,2,...,n\}$ into $M$ clusters $\{\mathcal{S}_1,\mathcal{S}_2,..., \mathcal{S}_M\}$. The learning process of FedCluster with local SGD is presented in \Cref{algo: fedcluster}. To elaborate, in each learning round of FedCluster, we first {sample a subset of devices from each cluster to be activated.} 
Then, in the inner cycles of this learning round, the clusters sequentially perform federated learning using the FedAvg algorithm. Specifically in each cycle, the cloud server sends the current model to {the activated devices of} the cluster to initialize their local models. Then, {these} devices perform multiple local SGD iterations in parallel and send the trained local models to the cloud for model averaging. After that, the updated global model is sent to {the activated devices of} the cluster in the next cycle. 

\begin{algorithm}
	\caption{FedCluster with local SGD}\label{algo: fedcluster}
	\label{alg: 1}
	{\bf Input:} Initialization model $W_0 \in \mathbb{R}^d$, learning rate $\eta_{j,K,t}$.
	
	\For{\normalfont{rounds} $j=0,1, \ldots, T-1$}{
	
		\For{\normalfont{cycles} $K=0,1,...,M-1$ }{
			Sample a subset of devices $S_{K+1}^{(j)}$ from cluster $\mathcal{S}_{K+1}$.
			Cloud \text{sends} $W_{jM+K}$ to the sampled devices.\\
			\For{\normalfont{all devices} $k \in S_{K+1}^{(j)}$ \normalfont{\textbf{in parallel}}}{
				\text{Initialize} $w_{j,K,0}^{(k)}=W_{jM+K}$.\\
				\For{$t=0,\ldots,E-1$}{
					Sample a local data point $\xi_{j,K,t}^{(k)} \in \mathcal{D}^{(k)}$ uniformly at random.	Update\\ $w_{j,K,t+1}^{(k)}=w_{j,K,t}^{(k)}-\eta_{j,K,t} \nabla f(w_{j,K,t}^{(k)}; \xi_{j,K,t}^{(k)})$.
				}
				{Send the local model $w_{j,K,E}^{(k)}$ to the cloud.}
			}
			Cloud computes \\
			$W_{jM+K+1}= \sum_{k\in \mathcal{S}_{j, \sigma_j(K+1)}'} \frac{p_k}{q_{\sigma_j(K+1)}} w_{j,K,E}^{(k)}$.
		}	
	}
	\textbf{Output:} $W_{TM}$
\end{algorithm}

We adopt the following standard assumptions on the loss function of the problem (P) \cite{li2019convergence}.
\begin{assum}\label{assum: P}
	The loss function in the problem (P) satisfies the following conditions.
	\begin{enumerate}[leftmargin=*, itemsep=0pt]
		\item For any data point $\xi$, function $f(\cdot; \xi)$ is $L$-smooth and bounded below. 
		\item For any $w\in \mathbb{R}^d$, the variance of stochastic gradients sampled by each device $k$ is bounded by $s_k^2$, i.e., 
		\begin{align*}
			\mathbb{E}_{\xi\sim \mathcal{D}^{(k)}} \|\nabla f(w;\xi) - \nabla f(w;\mathcal{D}^{(k)})\|^2 \le s_k^2.
		\end{align*}
		\item There exists $G>0$ such that for all $w_{j,K,t}^{(k)}$ and $\overline{w}_{j,K,t}$ (defined in \eqref{w_avg}) generated by \Cref{alg: 1} , it holds that
		$\mathbb{E}_{\xi\sim \mathcal{D}^{(k)}} \|\nabla f(w_{j,K,t}^{(k)};\xi)\|^2 \le G^2$ and $\mathbb{E}_{\xi\sim \mathcal{D}^{(k)}} \|\nabla f(\overline{w}_{j,K,t};\xi)\|^2 \le G^2$. 
	\end{enumerate}
\end{assum}

{We note that the item 3 of \Cref{assum: P} directly implies that for any $j=0,1,\ldots,T-1$, $K, K'=0,1,\ldots,M-1$, $t=0,1,\ldots,E-1$}
\begin{gather}
\max\big\{\|\nabla f(\overline{w}_{j,K,t}; \mathcal{D}^{(\mathcal{S}_{K'})})\|,\|\nabla f(\overline{w}_{j,K,t})\|\big\}\le G \label{df_G}.
\end{gather}


\subsection{Convergence {Result}}

In this subsection, we analyze the convergence rate of FedCluster with local SGD specified in \Cref{alg: 1} under {nonconvexity} of the loss function in the problem (P). In particular, we consider the simplified full participation setup, in which all the devices are activated in each learning round (i.e., $S_{K+1}^{(j)}=\mathcal{S}_{K+1}$). To simplified the analysis, we assume the clusters are chosen cyclically with reshuffle in each round to perform federated learning. 

Throughout the analysis, we define $q_K := \sum_{k\in \mathcal{S}_K} p_k$ and $f(w; \mathcal{D}^{(\mathcal{S}_K)}) := q_K^{-1} \sum_{k\in \mathcal{S}_K} p_k f(w; \mathcal{D}^{(k)})$, which characterizes the total loss of the cluster $\mathcal{S}_K$. 
Then, we adopt the following definition of cluster-level data heterogeneity that corresponds to the variance of the gradient on the local data possessed by the clusters. 
\begin{align}\label{H_cluster}
H_{\text{cluster}}:=\sup_{w\in \mathbb{R}^d} \Big( \mathbb{E} \sum_{K=1}^{M} q_K \|\nabla f(w; \mathcal{D}^{(\mathcal{S}_K)})-\nabla f(w)\|^2 \Big).
\end{align}

We obtain the following convergence result of FedCluster with local SGD in the nonconvex setting. Please refer to the appendix for the details of the proof.
\begin{thm}\label{thm: nc}
	Let \Cref{assum: P} hold and assume that $f(\cdot;\xi)$ is nonconvex for any data sample $\xi$. Choose learning rate $\eta_{j,K,t}\equiv(TME)^{-\frac{1}{2}}$ and choose $E,M,T$ such that $ME \le \frac{C}{8LG^2}$, $T\ge L^2\max(1,  \frac{16}{EM})$. Then, under full participation of the devices, the output of \Cref{alg: 1} satisfies
	\begin{align}\label{err_nc}
	\frac{1}{T}\sum_{j=0}^{T-1} \mathbb{E}\Big\| \nabla f(W_{jM}) \Big\|^2 \le& \frac{2C}{\sqrt{TME}},
	\end{align}
	where the constant $C$ is defined as 
	\begin{align}\label{C}
	C=&2\mathbb{E}\big(f(W_0)-\inf_{w\in \mathbb{R}^d} f(w)\big)\nonumber\\
	&+4L\Big(H_{\text{cluster}}+\sum_{K=1}^M q_K^{-1} \sum_{k\in \mathcal{S}_K} p_k^2 s_k^2\Big).
	\end{align}
	Furthermore, in order to achieve a solution that satisfies $\frac{1}{T}\sum_{j=0}^{T-1} \mathbb{E}\| \nabla f(W_{jM}) \|^2 \le \epsilon$, the required number of learning rounds satisfies $T \propto \frac{C^2}{ME}$. 
\end{thm}

Therefore, under cluster-cycling, FedCluster with local SGD enjoys a convergence rate $\mathcal{O}(\frac{1}{\sqrt{TME}})$ in nonconvex optimization, which is faster than the convergence rate  $\mathcal{O}(\frac{1}{\sqrt{TE}})$ of the FedAvg algorithm \cite{li2019communication}. Also, the above convergence rate of FedCluster depends on the cluster-level data heterogeneity $H_{\text{cluster}}$, whereas the convergence rate of the FedAvg algorithm depends on the larger device-level data heterogeneity $H_{\text{device}}:= \sup_{w\in \mathbb{R}^d} ( \mathbb{E}\sum_{k=1}^{n} p_k \|\nabla f(w; \mathcal{D}^{(k)})-\nabla f(w)\|^2 )$ (It is easy to show that $H_{\text{cluster}}\le H_{\text{device}}$). {The complexity result \eqref{err_nc} also implies a trade-off between the number of clusters $M$ and the constant $C$. Specifically, with more clusters (i.e., a larger $M$), it can be shown that both the cluster-level data heterogeneity $H_{\text{cluster}}$ and the local variance $\sum_{K=1}^M q_K^{-1} \sum_{k\in \mathcal{S}_K} p_k^2 s_k^2$ in $C$ increase. Hence, a proper choice of $M$ that minimizes $\frac{C}{\sqrt{M}}$ can yield the fastest convergence rate in nonconvex case. On the other hand, given a fixed number of clusters $M$, a proper clustering approach that minimizes the cluster-level data heterogeneity $H_{\text{cluster}}$ can also improve the convergence rate.} We further explore the impact of these factors on the convergence speed of FedCluster in the following experiment section.

%% file: experiments.tex
\section{Experiments}\label{exp}
\subsection{Experiment Setup}
In this section, we compare the performance of FedCluster with that of the traditional federated learning in deep learning applications. We consider completing a standard classification tasks on two datasets -- CIFAR-10 \cite{Krizhevsky09} and MNIST \cite{Lecun_1998}, using the {AlexNet} model \cite{Zhang_2017} and the cross-entropy loss. 
{We simulate 1000 devices for both FedCluster and the traditional federated learning system, and each device possesses 500 data samples. }
Specifically, the dataset of each device is specified by a major class and a device-level data heterogeneity ratio $\rho_{\text{device}}\in[0.1, 1]$. Take the CIFAR-10 dataset as an example, each of its ten classes is assigned as the major class of 100 devices. Then, $\rho_{\text{device}}\times 100\%$ of the samples of each device are sampled from the major class, and $(1-\rho_{\text{device}})/9\times 100\%$ of the samples are sampled from each of the other classes. Hence, a larger $\rho_{\text{device}}$ corresponds to a higher device-level data heterogeneity. For FedCluster, by default we cluster the devices into 10 clusters uniformly at random.

For the traditional federated learning system, in every learning round we randomly activate 10\% of the devices to participate in the training. By default, these activated devices run $E=20$ local SGD steps with batch size 30 using fine-tuned learning rates $\eta_t\equiv 0.1$ and $0.05$ for CIFAR-10 and MNIST, respectively. For the FedCluster system, in every learning cycle we randomly activate 10\% of the devices of the corresponding cluster to participate in the training. These activated devices run $E=20$ local SGD steps with batch size 30 using the learning rates $\eta_t\equiv 0.01$ and $0.005$ for CIFAR-10 and MNIST, respectively. In particular, to make a fair comparison, these choices of learning rates are one tenth of those adopted by the traditional federated learning system, as in each cycle only one of the ten clusters is involved (hence less number of data samples are used). We note that the learning rates adopted by FedCluster are not fine-tuned. We also implement a centralized SGD as a baseline, which adopts 1000 iterations per learning round, batch size 60 and fine-tuned learning rates $0.01$ and $0.005$ for CIFAR-10 and MNIST, respectively.
This ensures that the federated learning algorithms and the centralized SGD consume the same number of samples (i.e., 60k training samples) per learning round. All experiments with a given model and dataset are initialized at the same point. 


{\subsection{Comparison under Different Device-level Data Heterogeneities}\label{exp_1}}
We first compare FedCluster (with local SGD) with the conventional FedAvg under different levels of device-level data heterogeneity ratios, i.e., $\rho_{\text{device}}=0.1, 0.4, 0.7, 0.9$. 
We present the train loss and test accuracy results on CIFAR-10 in Fig. \ref{fig: 2}, where the top row presents the results of $\rho_{\text{device}}=0.1, 0.4$ and the bottom row presents the results of $\rho_{\text{device}}=0.7, 0.9$. It can be seen that FedCluster achieves faster convergence than FedAvg in terms of both train loss and test accuracy under different levels of $\rho_{\text{device}}$, which demonstrates the advantage of FedCluster in both training efficiency and generalization ability with heterogeneous data. 

\begin{figure}[htbp]
	\vspace{-2mm}
	\centering
	\includegraphics[width=0.24\textwidth,height=0.2\textwidth]{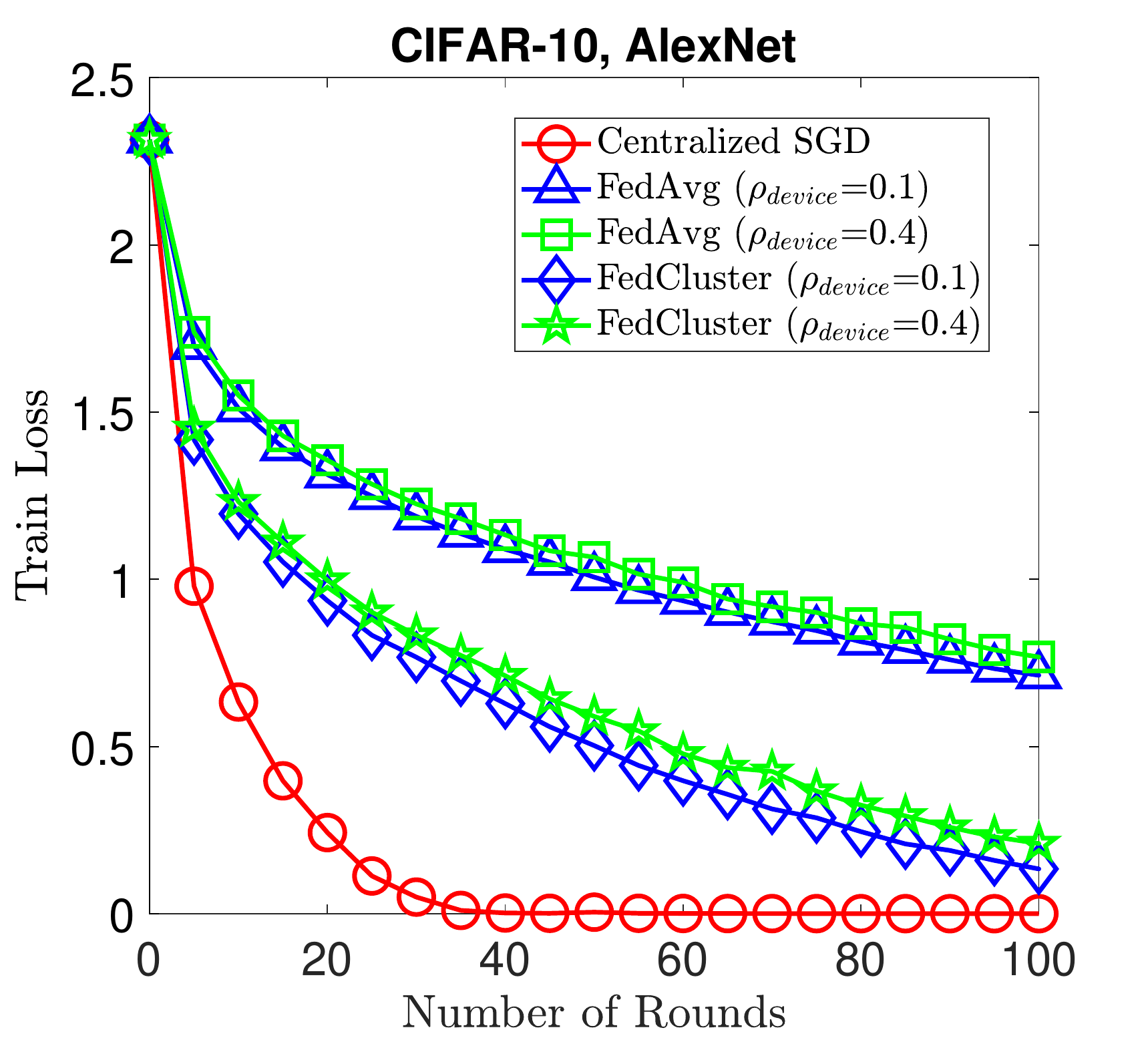}	
	\includegraphics[width=0.24\textwidth,height=0.2\textwidth]{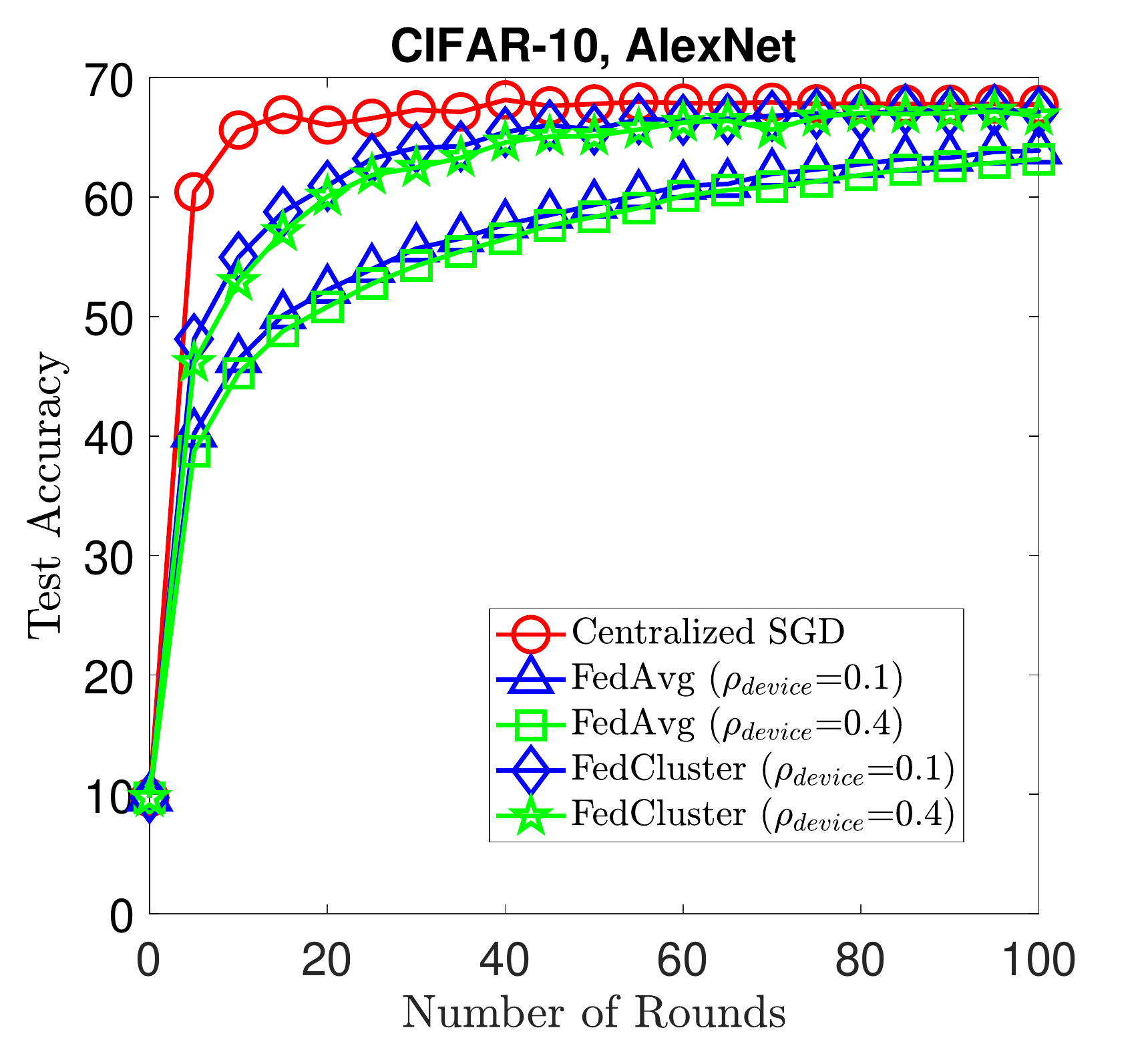}	
	\includegraphics[width=0.24\textwidth,height=0.2\textwidth]{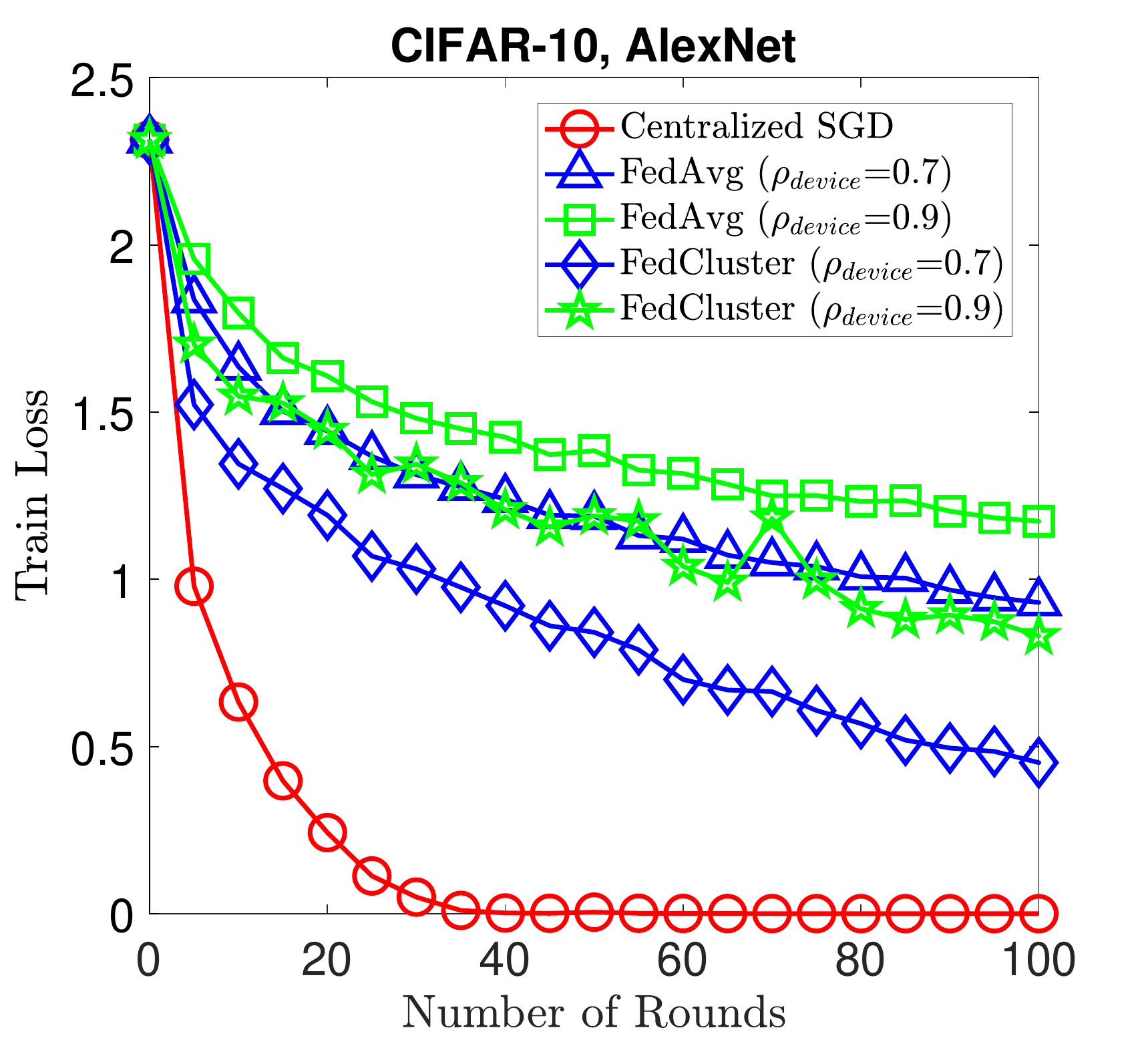}
	\includegraphics[width=0.24\textwidth,height=0.2\textwidth]{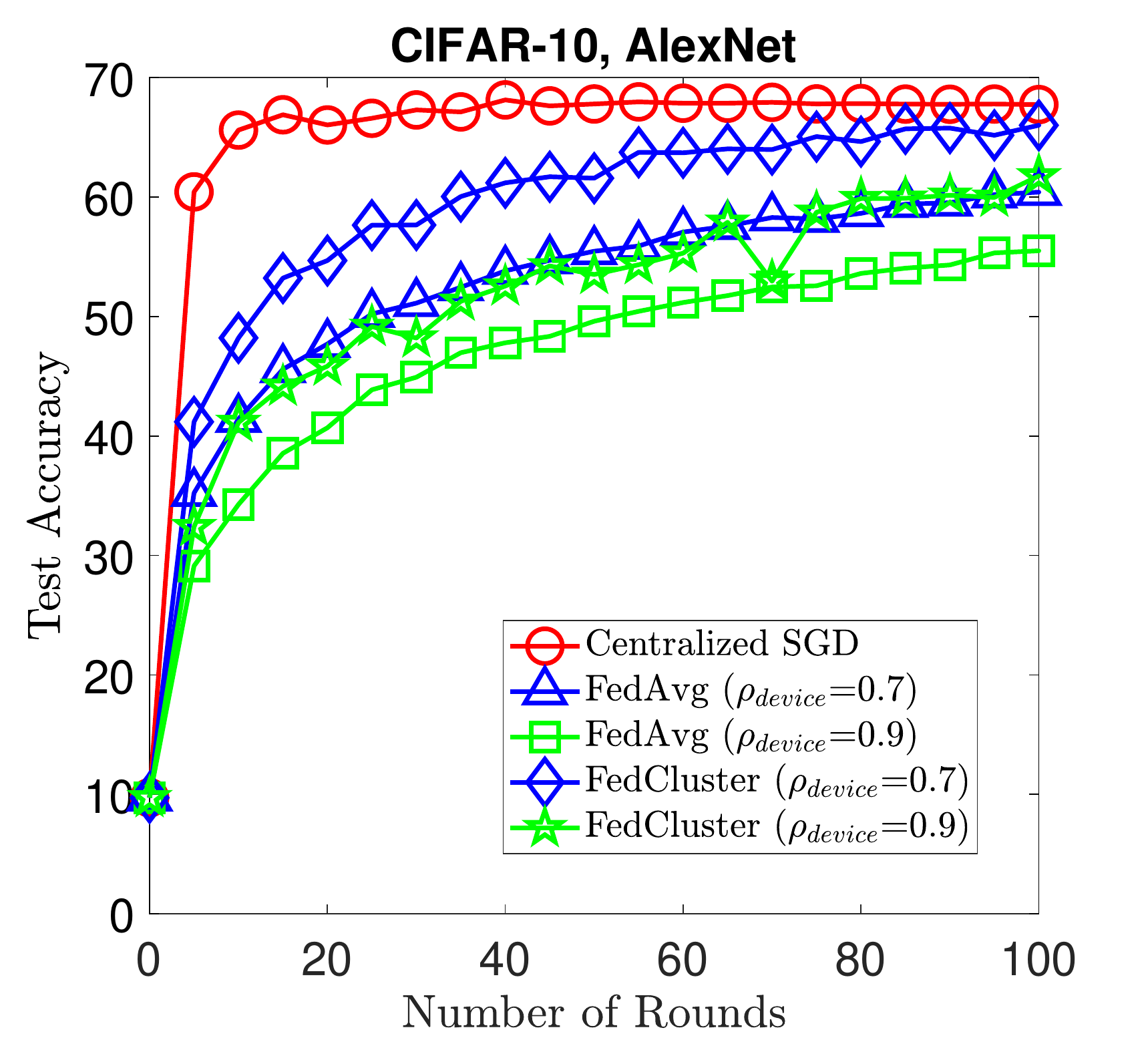}
	\vspace{-6mm}
	\caption{\small{Comparison between FedCluster and FedAvg under different device-level data heterogeneities on CIFAR-10.}}\label{fig: 2}
	\vspace{-2mm}
\end{figure}
In the following Fig. \ref{fig: 22}, we present the train loss and test accuracy results on MNIST, and one can make similar observations as above. In particular, by comparing Fig. \ref{fig: 2} with Fig. \ref{fig: 22}, it seems that FedCluster is more advantageous when the data is more complex. 

\begin{figure}[htbp]
	\vspace{-2mm}
	\centering
		\includegraphics[width=0.24\textwidth,height=0.2\textwidth]{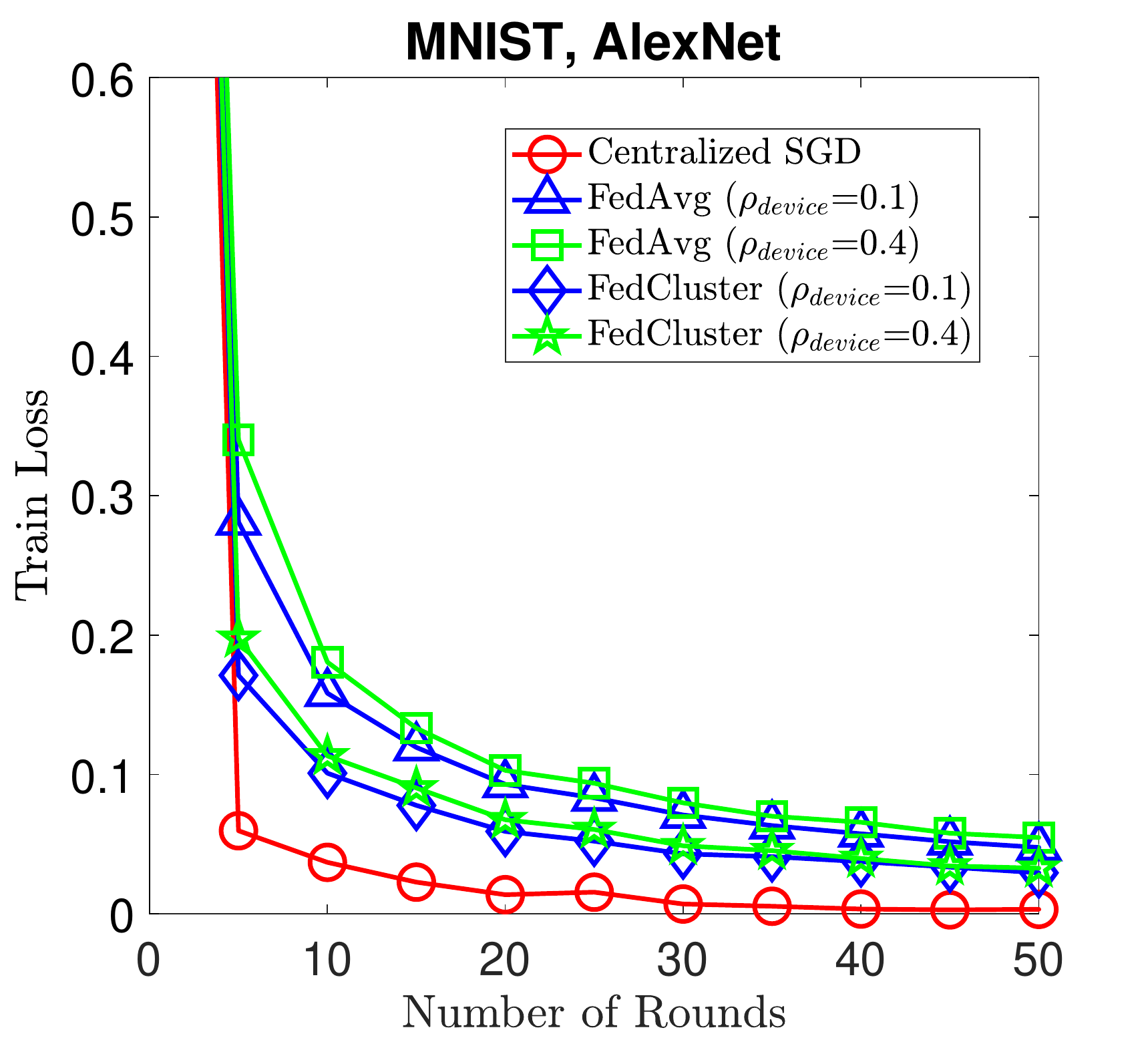}
		\includegraphics[width=0.24\textwidth,height=0.2\textwidth]{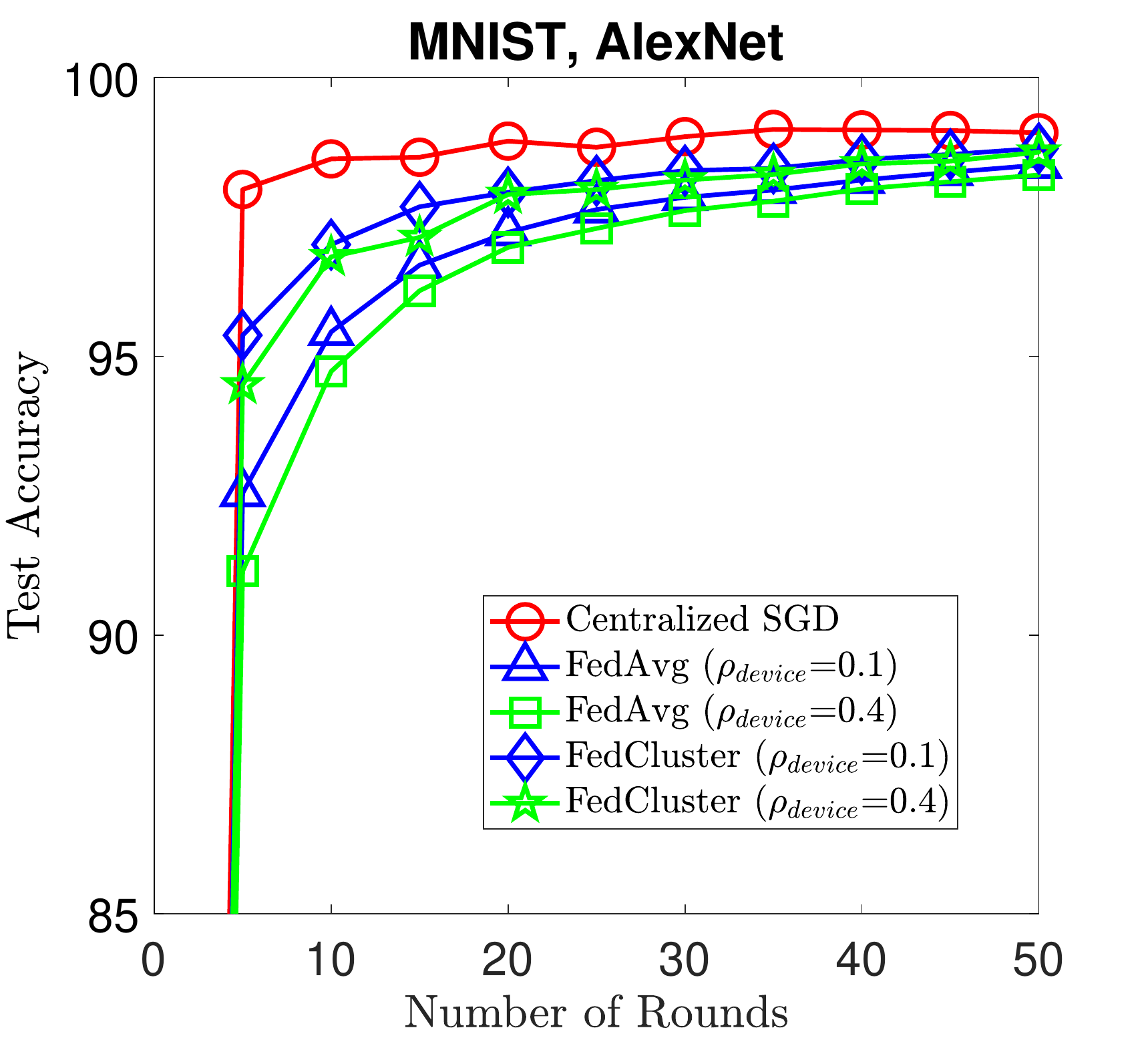}
		\includegraphics[width=0.24\textwidth,height=0.2\textwidth]{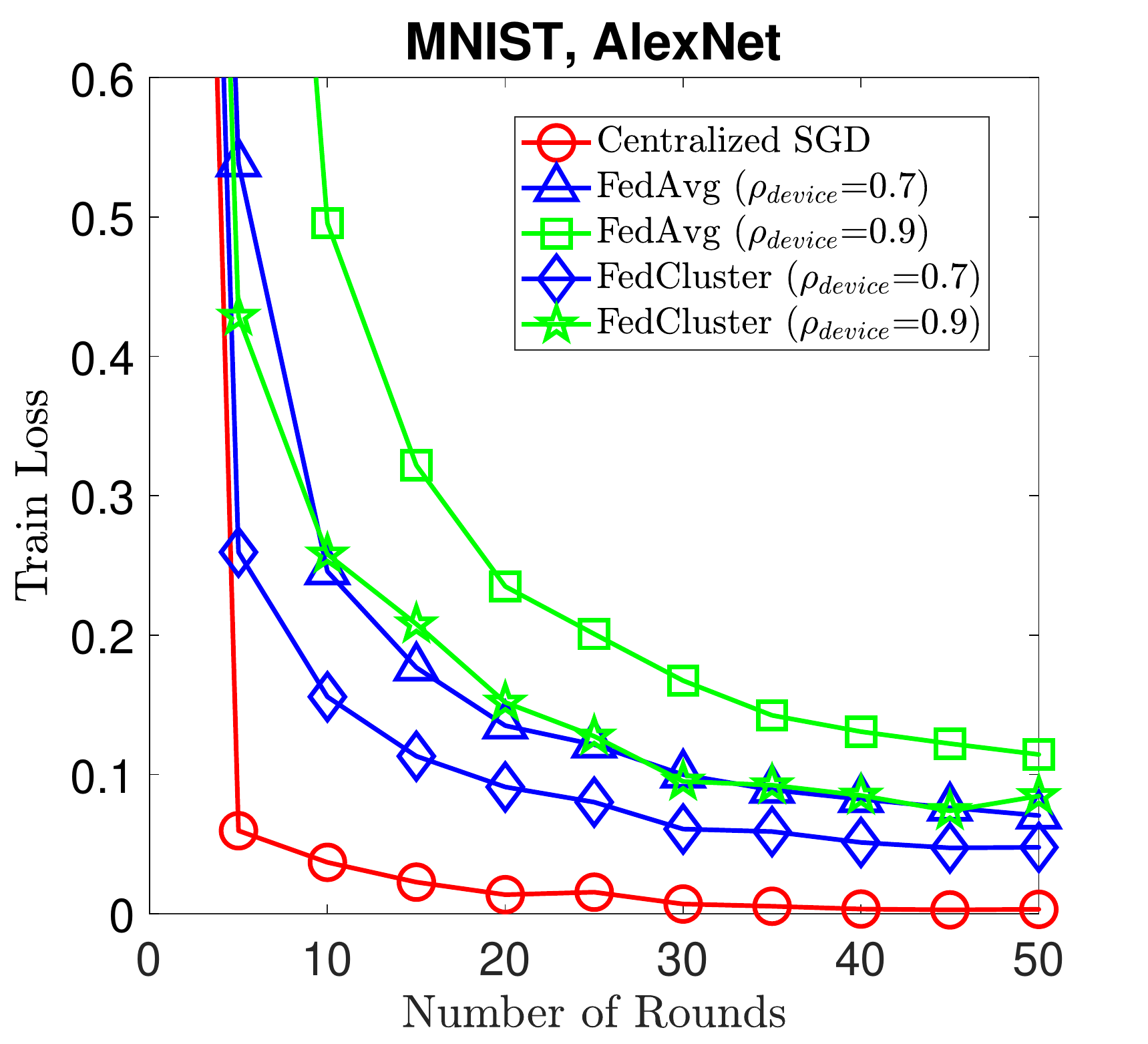}
		\includegraphics[width=0.24\textwidth,height=0.2\textwidth]{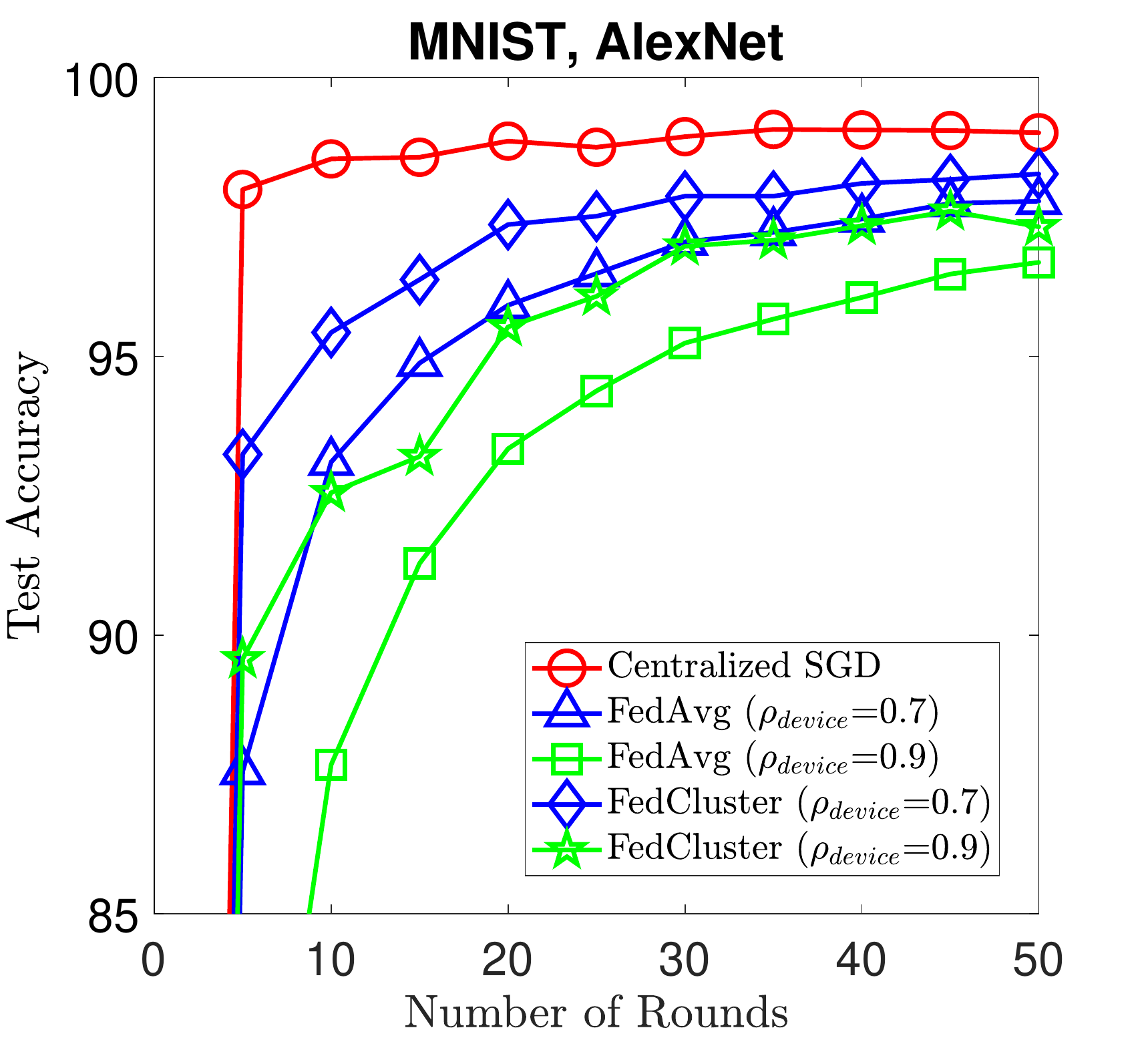}
	\vspace{-6mm}
	\caption{\small{Comparison between FedCluster and FedAvg under different device-level data heterogeneities on MNIST.}}\label{fig: 22}
	\vspace{-2mm}
\end{figure}

\vspace{-3pt}
\subsection{Comparison under Different Optimizers}\label{exp_2}
We further compare FedCluster and the traditional federated learning with different choices of local optimizers, including 1) SGD-momentum (SGDm) with $m$ = 0.5; 2) Adam with $(\beta_1, \beta_2)$ = (0.9, 0.999), $\epsilon$ = ${10^{-8}}$; and 3) FedProx with $\mu$ = 0.1. 
{The learning rate for Adam is chosen to be small enough to ensure convergence, while the other optimizers use the default learning rate. We set $\rho_{\text{device}} = 0.1, 0.5$. Fig. \ref{fig: 3} presents the training results with different local optimizers, where the first column shows the results on CIFAR-10 and the second column shows the results on MNIST. The testing results are similar to the training results and hence are omitted. It can be seen that FedCluster significantly outperforms the traditional federated learning under all choices of local optimizers and all levels of device-level data heterogeneity.} Again, FedCluster seems to be more advantageous when dealing with more complex datasets.

\begin{figure}[htbp]
	\centering	
	\vspace{-2mm}
	\includegraphics[width=0.24\textwidth,height=0.2\textwidth]{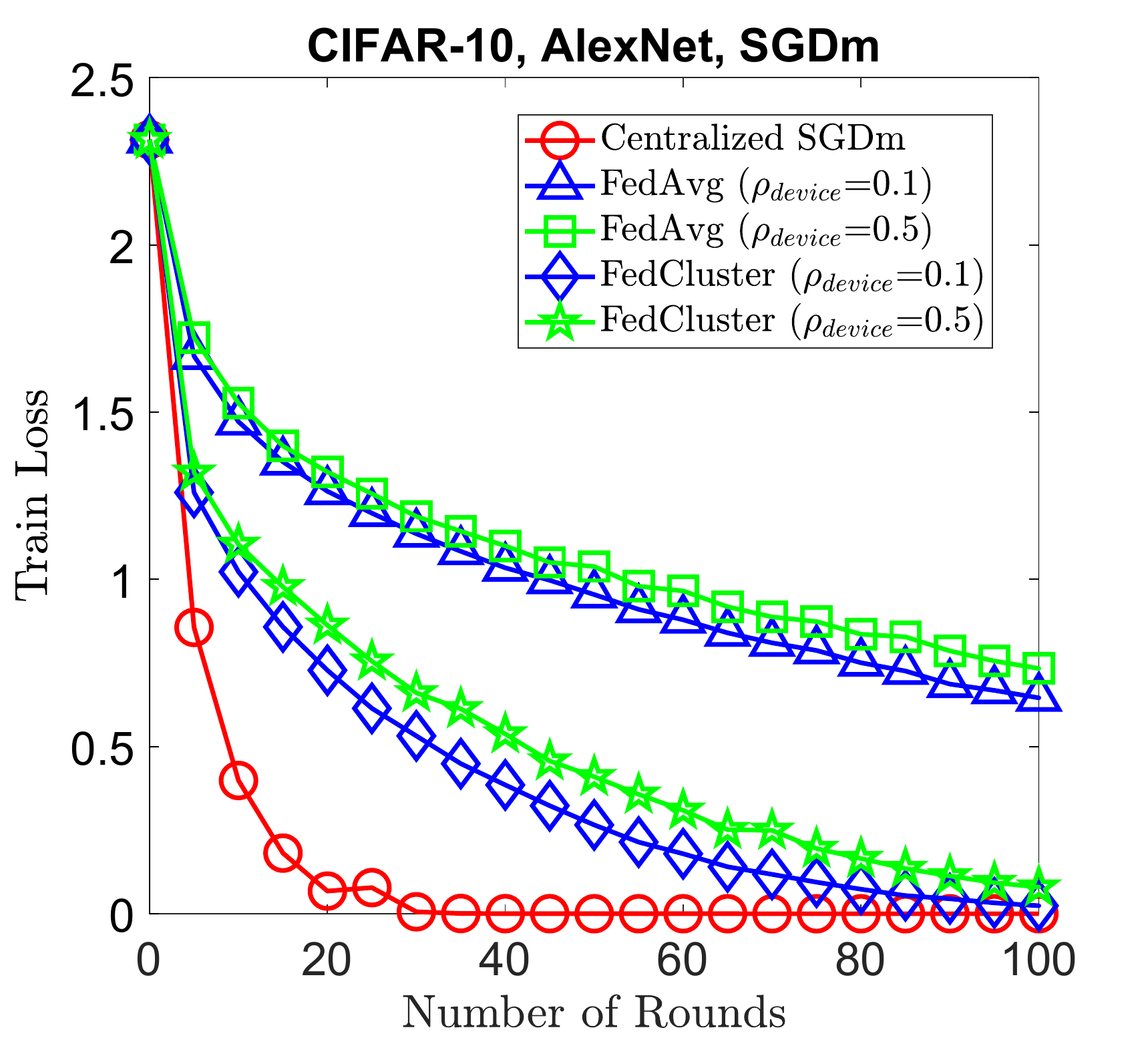}
	\includegraphics[width=0.24\textwidth,height=0.2\textwidth]{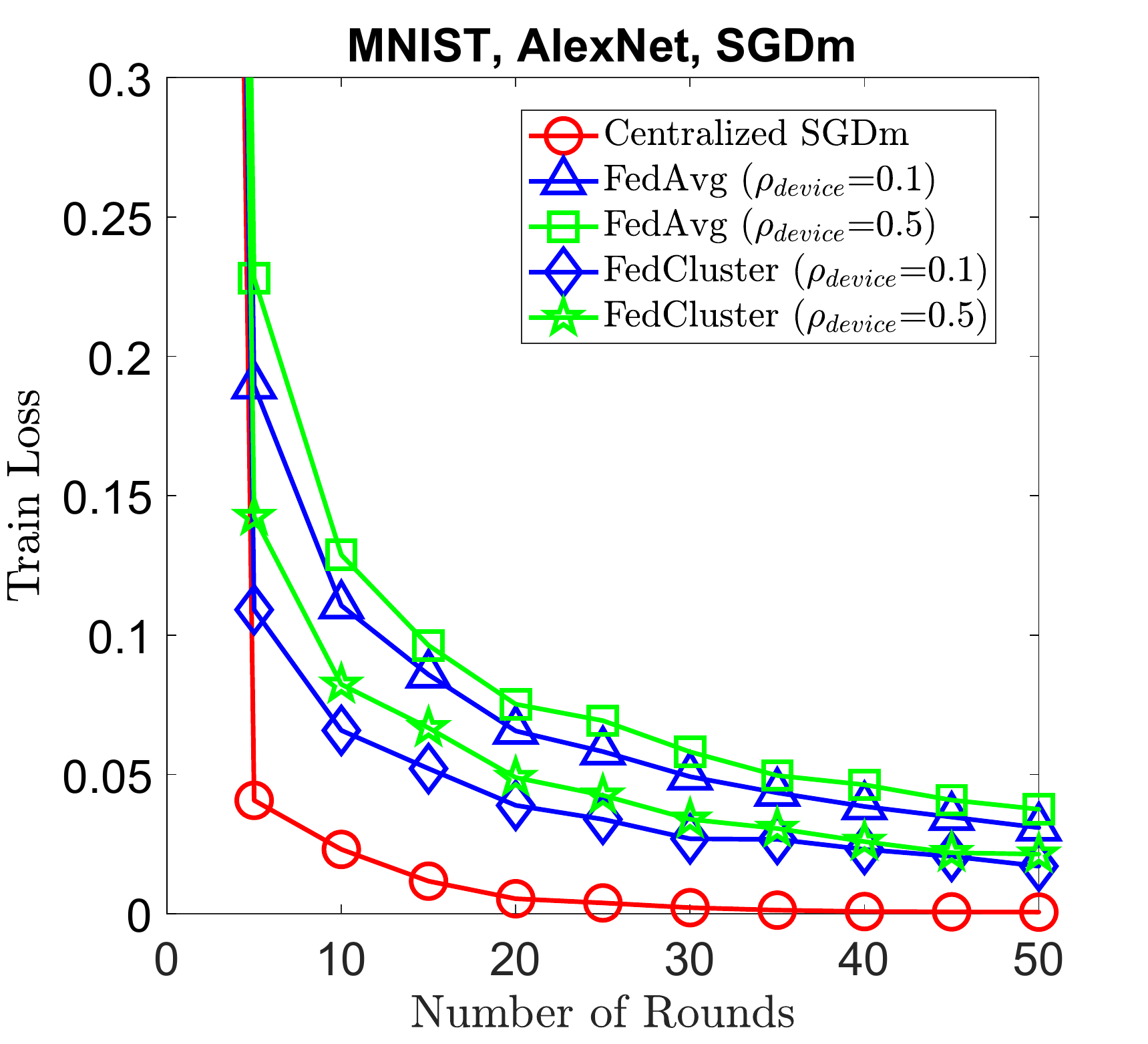}
	\includegraphics[width=0.24\textwidth,height=0.2\textwidth]{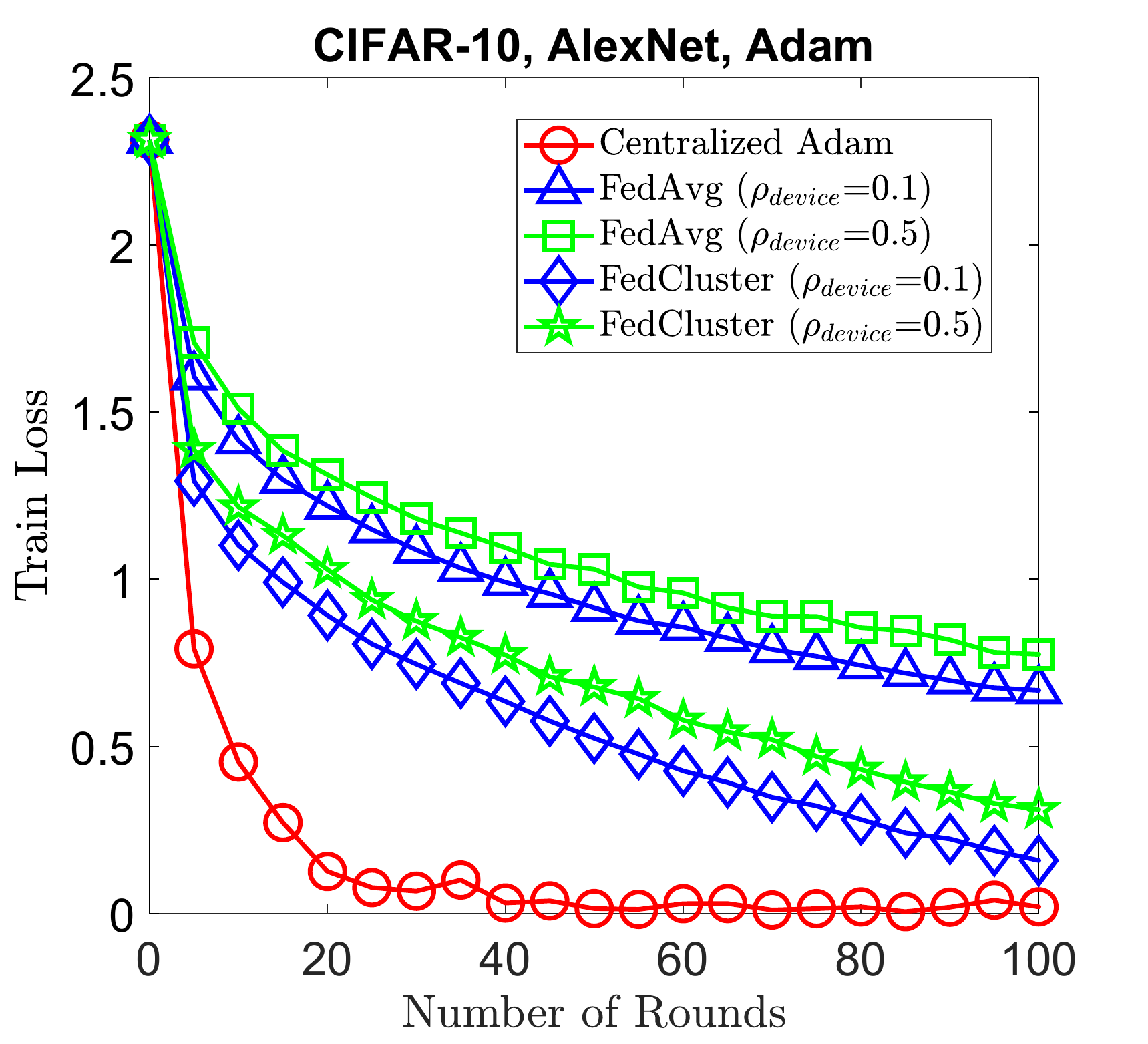}
	\includegraphics[width=0.24\textwidth,height=0.2\textwidth]{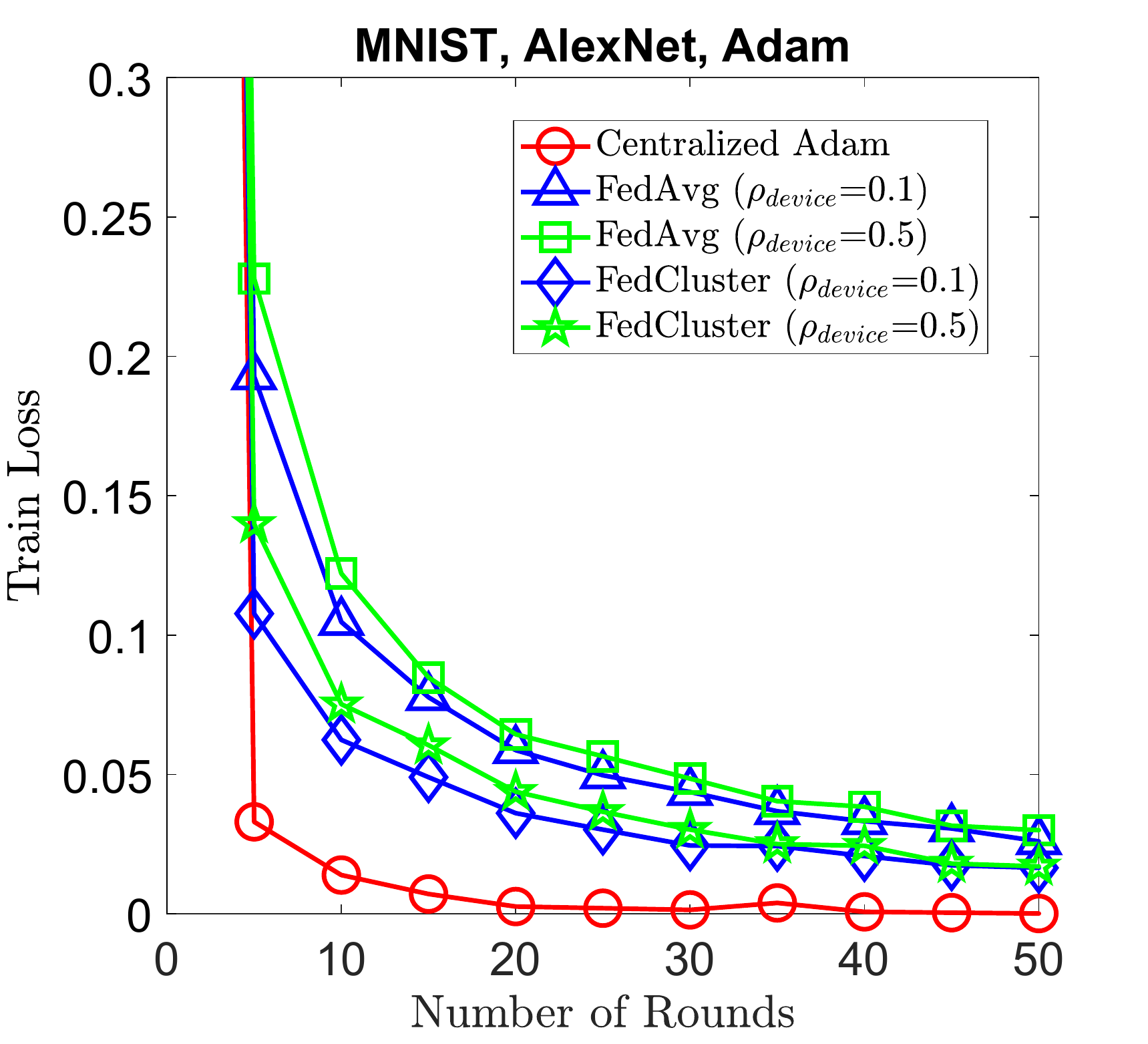}
	\includegraphics[width=0.24\textwidth,height=0.2\textwidth]{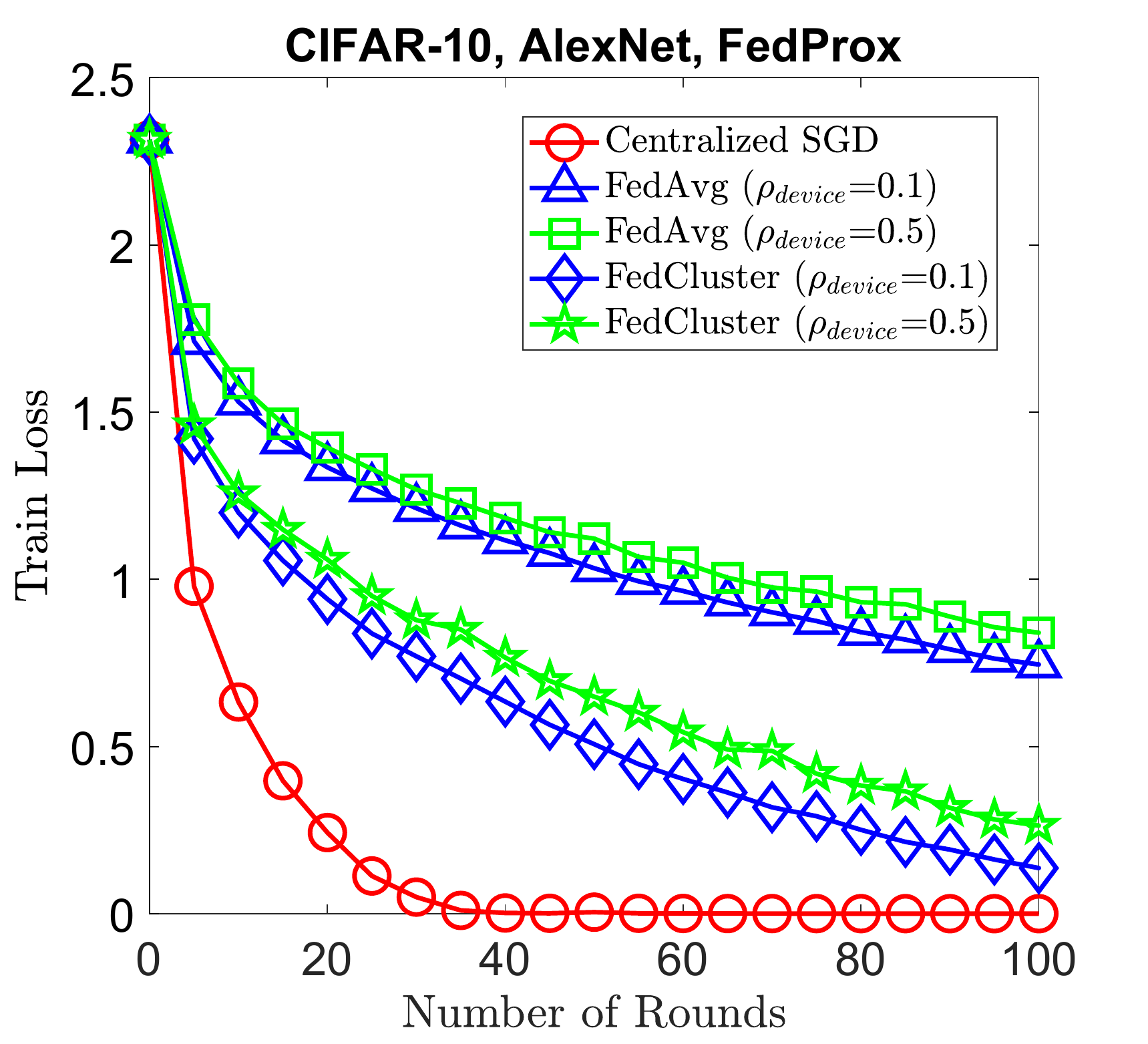}
	\includegraphics[width=0.24\textwidth,height=0.2\textwidth]{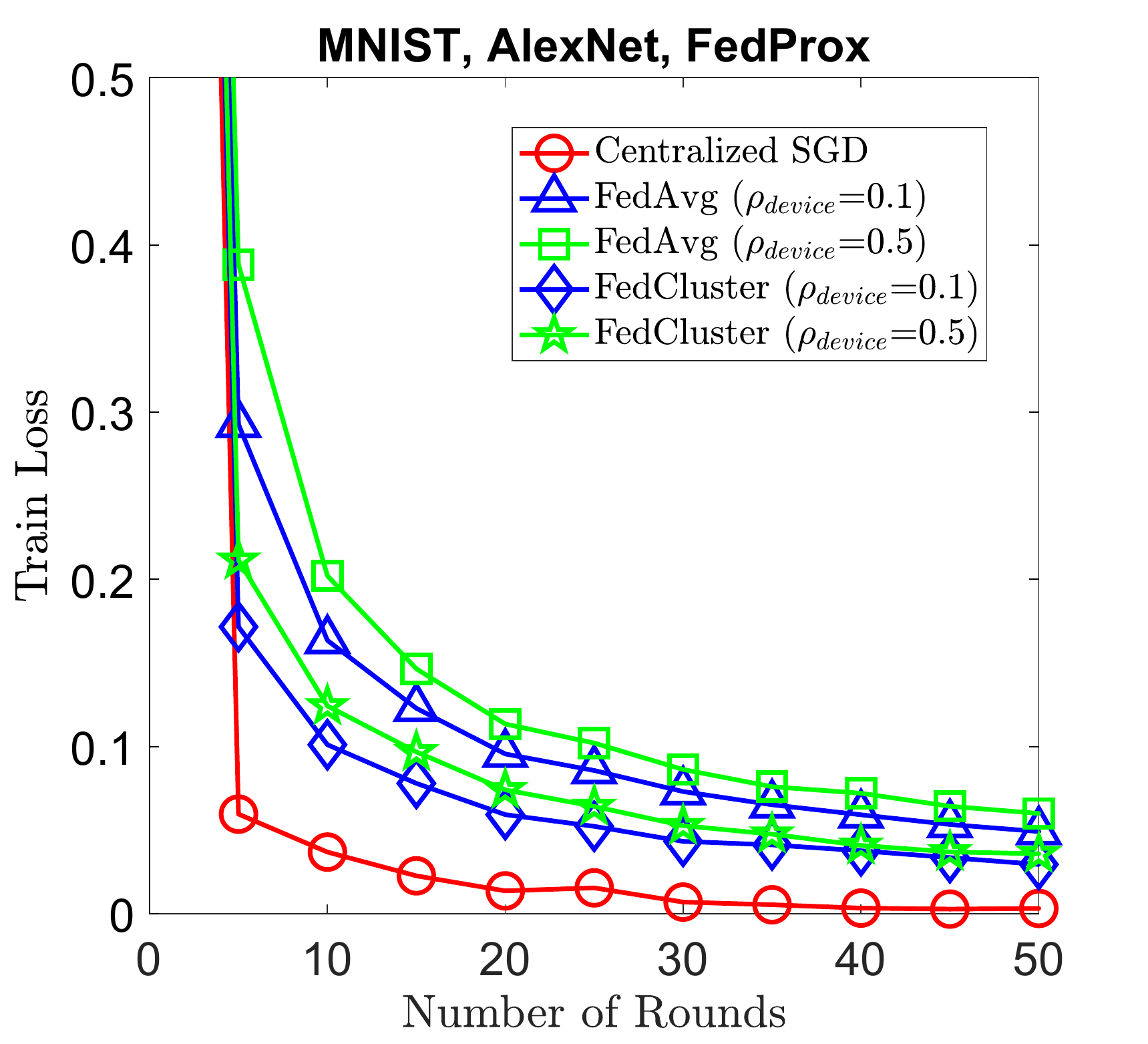}
	\vspace{-6mm}
	\caption{\small{Comparison between FedCluster and FedAvg under different local optimizers on CIFAR-10 (left) and MNIST (right).}}\label{fig: 3}
\end{figure}

\vspace{-3pt}
\subsection{Effect of Number of Clusters}\label{exp_3}
We further explore the impact of the number of clusters on the performance of FedCluster with local SGD. We set $\rho_{\text{device}}=0.1, 0.5$ and explore the number of clusters choices $M=5, 10, 20$. The {training} results on CIFAR-10 {and MNIST} are presented in the first and second row of Fig. \ref{fig: 4}, respectively. From the top row, it can be seen that FedCluster outperforms FedAvg under all choices of number of clusters on the complex CIFAR-10 dataset. In particular, a larger number of clusters leads to faster convergence of FedCluster, as indicated by \Cref{thm: nc}. From the bottom row, similar observations can be made on the MNIST dataset except that FedCluster with 5 clusters has comparable performance to FedAvg. In practice, the number of clusters of FedCluster is limited by the number of learning cycles allowed within a learning round.

\begin{figure}[htbp]
	\centering
	\includegraphics[width=0.24\textwidth,height=0.2\textwidth]{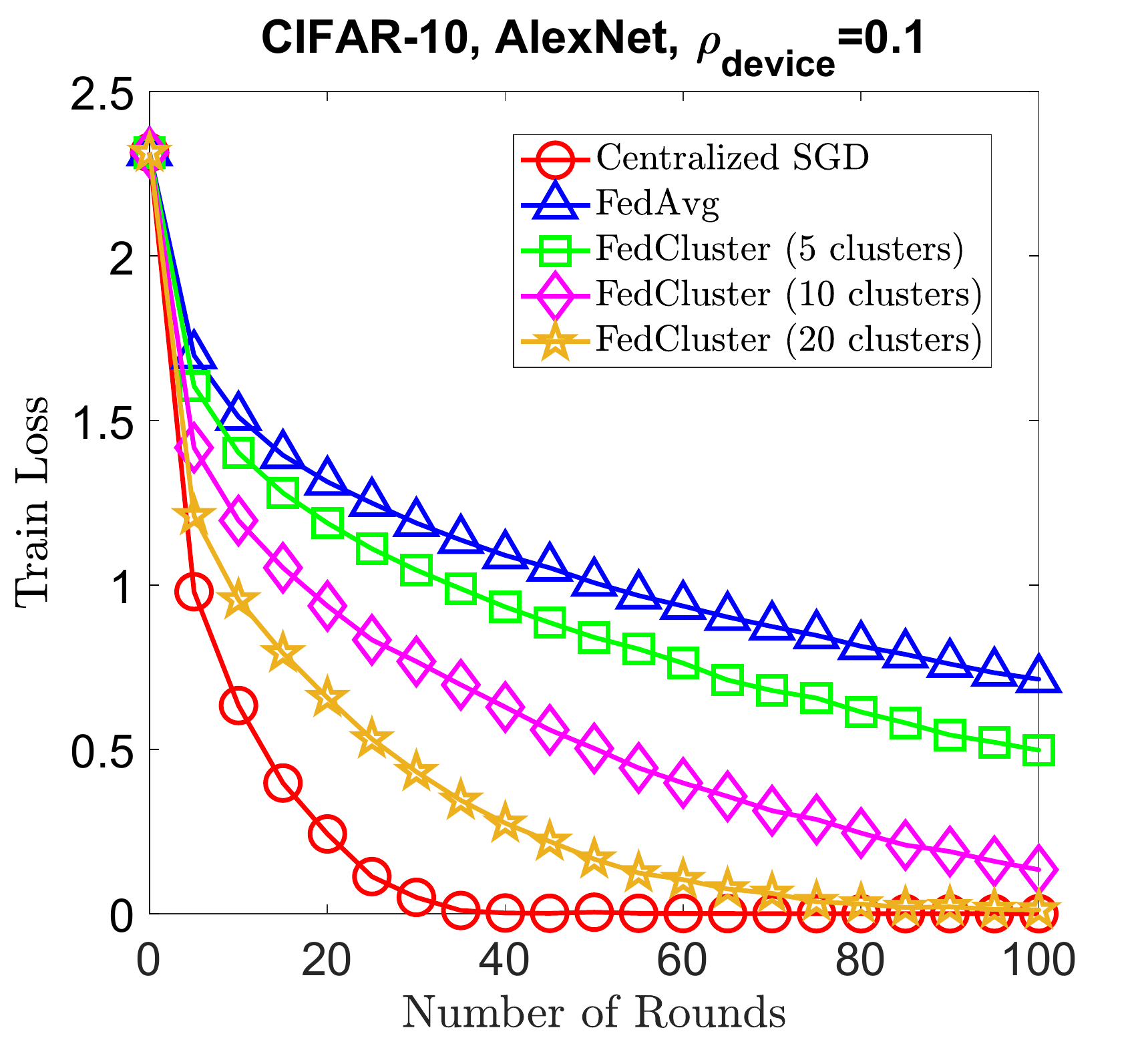}	
	\includegraphics[width=0.24\textwidth,height=0.2\textwidth]{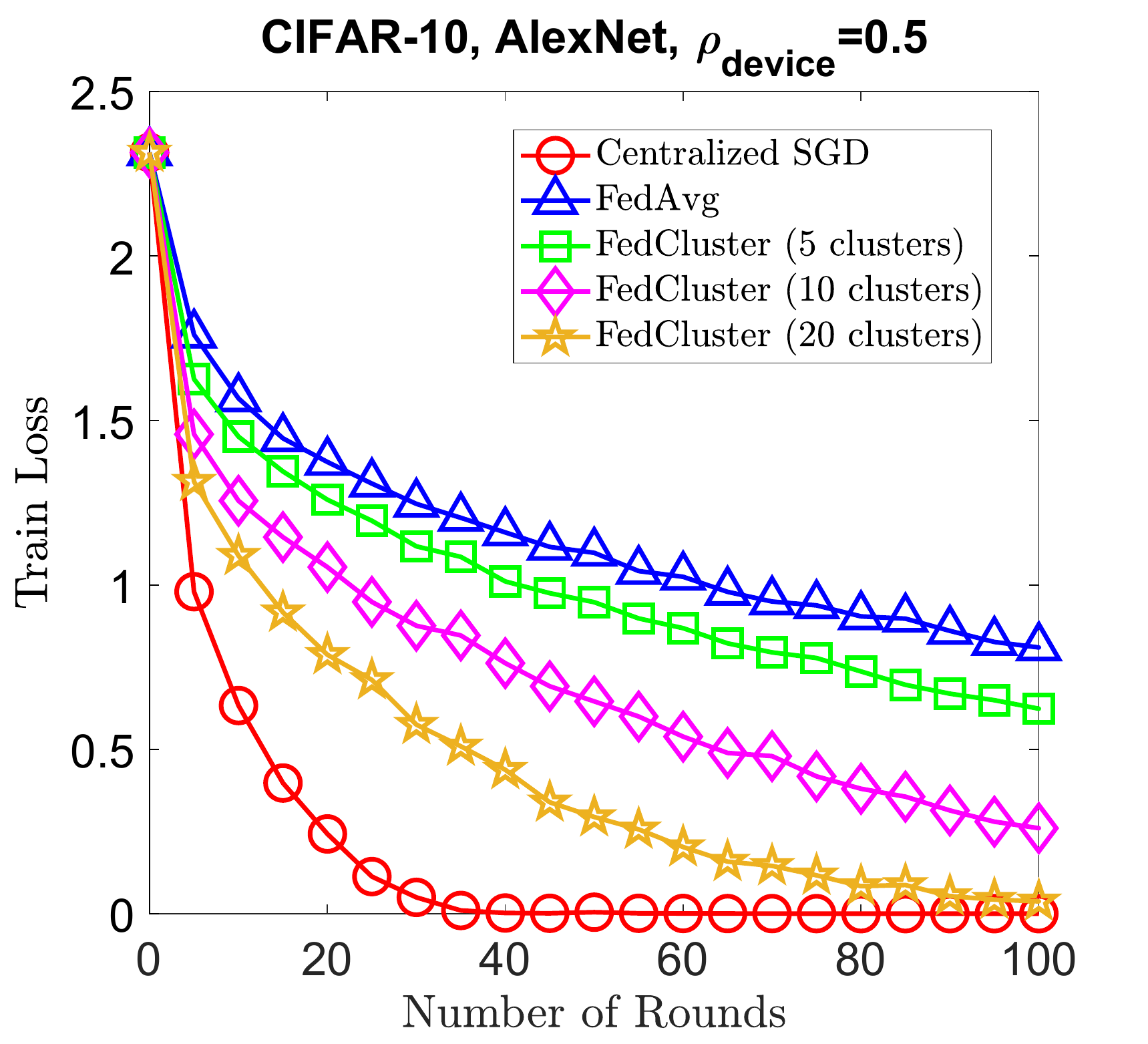}
	\includegraphics[width=0.24\textwidth,height=0.2\textwidth]{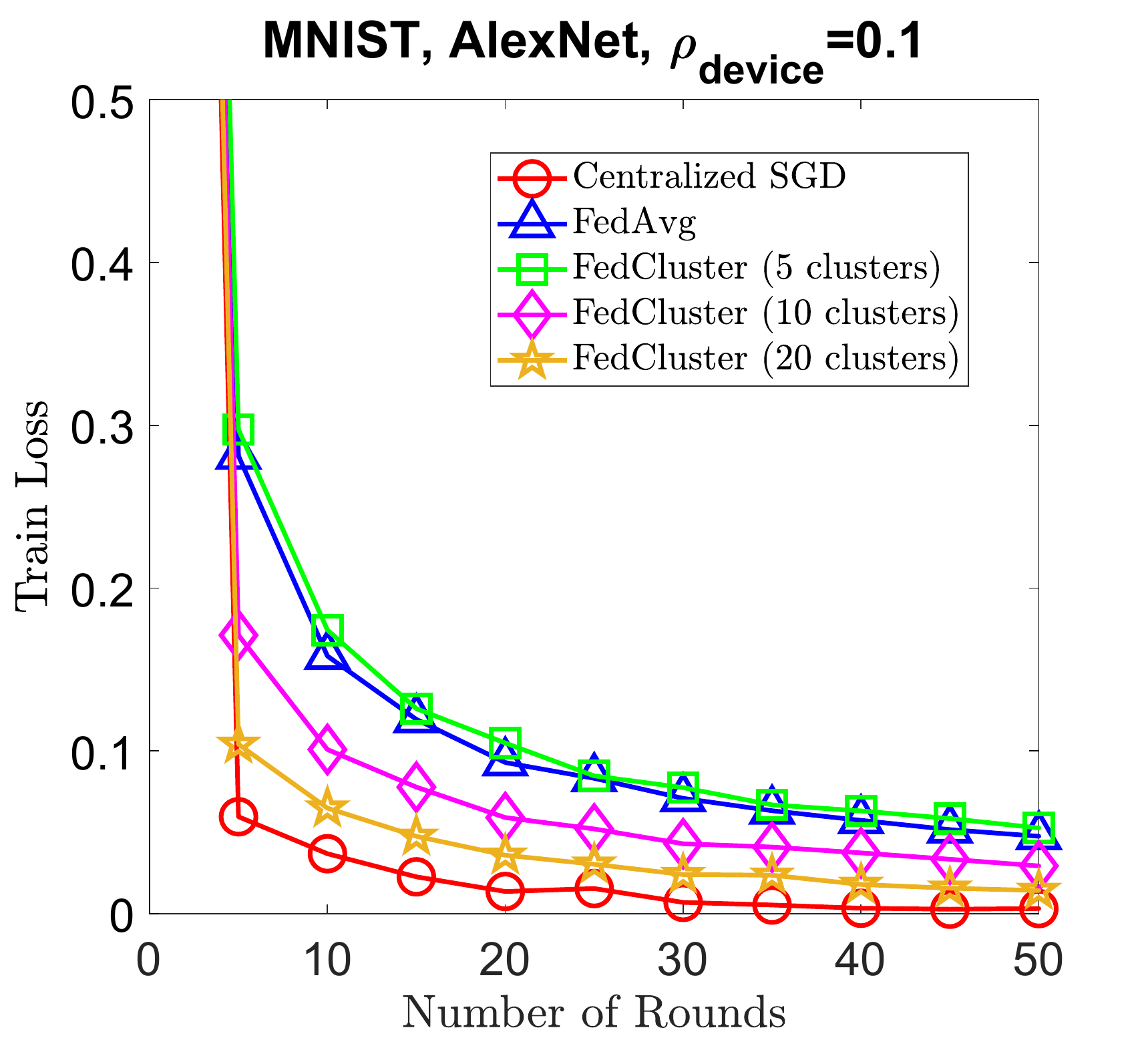}
	\includegraphics[width=0.24\textwidth,height=0.2\textwidth]{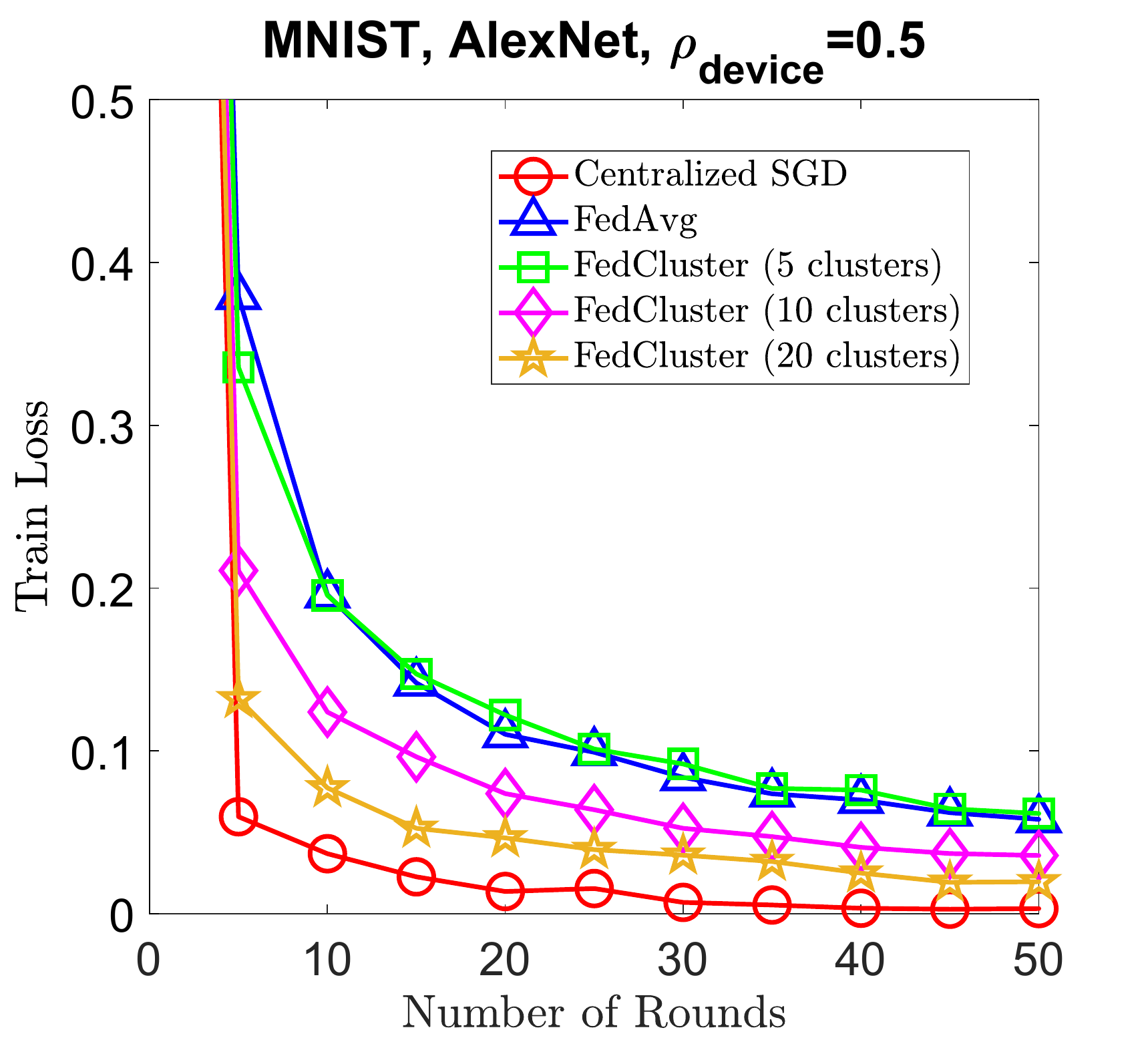}
	\vspace{-6mm}
	\caption{\small{Comparison between FedCluster and FedAvg under different number of clusters on CIFAR-10 (top) and MNIST (bottom).}}\label{fig: 4}
\end{figure}

\vspace{-6mm}
\subsection{Comparison under Different Cluster-level Data Heterogeneities}\label{exp_4}
We introduce an additional cluster-level data heterogeneity ratio $\rho_{\text{cluster}}\in[0.1, 1]$ to specify how the devices of FedCluster are clustered. Specifically, each cluster is assigned a different major class and $\rho_{\text{cluster}}\times 100\%$ of the devices in each cluster are assigned the same major class, whereas $(1-\rho_{\text{cluster}})/9\times 100\%$ of the devices in the cluster are assigned a different major class. Hence, a larger $\rho_{\text{cluster}}$ indicates a higher cluster-level data heterogeneity. 
We fix $\rho_{\text{device}}=0.5$ and compare FedCluster with FedAvg under different levels of cluster-level data heterogeneity, i.e., $\rho_{\text{cluster}}=0.1, 0.5, 0.9$. The training results on CIFAR-10 and MNIST are presented in Fig. \ref{fig: 5}. It can be seen that FedCluster achieves faster convergence than FedAvg under all levels of $\rho_{\text{cluster}}$. In addition, a smaller $\rho_{\text{cluster}}$ (i.e., lower cluster-level data heterogeneity) leads to slightly faster convergence of FedCluster, which is also indicated by \Cref{thm: nc}. 

\begin{figure}[htbp]
	\centering
		\vspace{-2mm}
	\includegraphics[width=0.24\textwidth,height=0.2\textwidth]{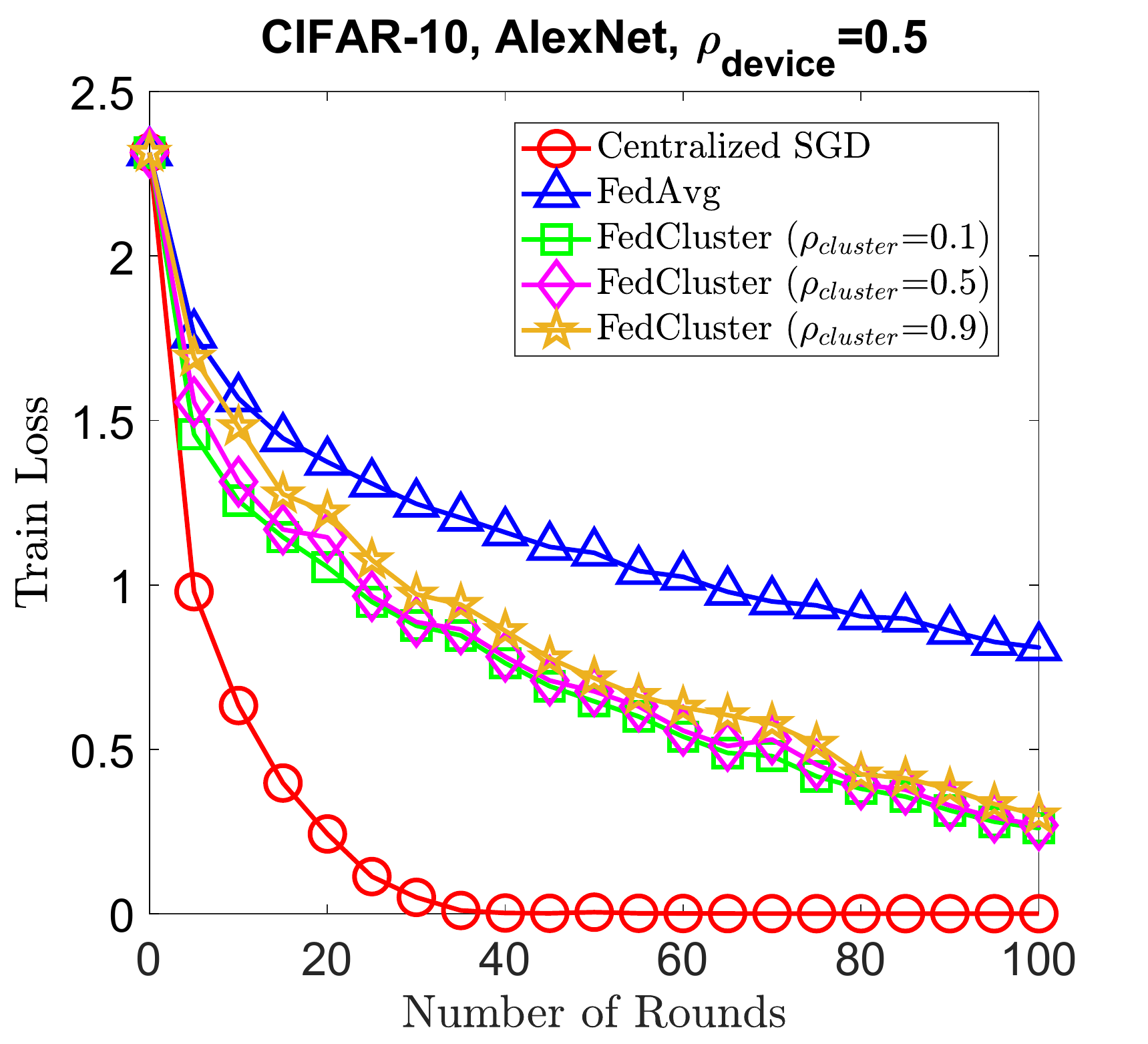}
	\includegraphics[width=0.24\textwidth,height=0.2\textwidth]{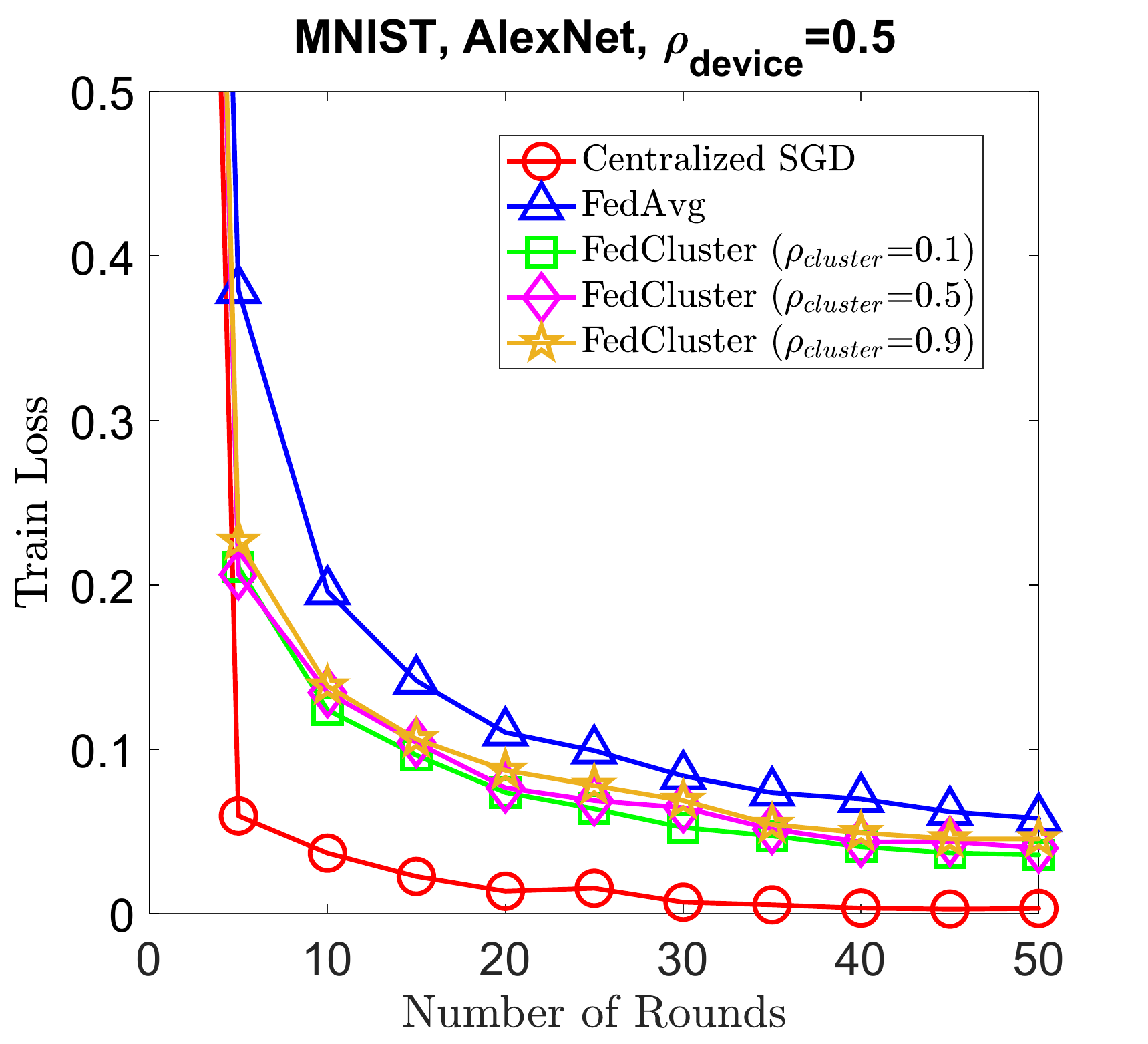}
	\vspace{-6mm}
	\caption{\small{Comparison between FedCluster and FedAvg under different cluster-level data heterogeneities on CIFAR-10 (left) and MNIST (right).}}\label{fig: 5}
\end{figure}

%% file: supp_exp.tex
\section{Supporting Lemmas}\label{sec_lemma}
{In this section, we present and prove some supporting lemmas for the convenience  of  proving the main Theorem \ref{thm: nc} in the next subsection.}

Throughout the analysis, we consider a random reshuffle of the clusters. To elaborate, for every $j$-th round, we define a random permutation $\sigma_{j}:\{1,\ldots,M\}\to\{1,\ldots,M\}$, and the reshuffled clusters $\mathcal{S}_{\sigma_j(1)}$, \ldots, $\mathcal{S}_{\sigma_j(M)}$ sequentially perform federated learning. For every $t$-th iteration of the $K$-th cycle in the $j$-th learning round, we define the following weighted average of the local models obtained by the devices in the active cluster $\mathcal{S}_{\sigma_j(K+1)}$.
\begin{align}\label{w_avg}
\overline{w}_{j,K,t}:=\sum_{k\in \mathcal{S}_{\sigma_{j}(K+1)}} \frac{p_k}{q_{\sigma_{j}(K+1)}} w_{j,K,t}^{(k)}.
\end{align}

Then, by the local update rule of the FedCluster in \Cref{alg: 1}, it holds that
\begin{align}\label{w_avg_update}
\overline{w}_{j,K,t+1}=\overline{w}_{j,K,t}-\eta_{j,K,t} g_{j,K,t},
\end{align}
where $g_{j,K,t}$ is defined as
\begin{align}\label{g}
g_{j,K,t}:=\sum_{k\in \mathcal{S}_{\sigma_{j}(K+1)}} \frac{p_k}{q_{\sigma_{j}(K+1)}} \nabla f(w_{j,K,t}^{(k)}; \xi_{j,K,t}^{(k)}).
\end{align}

We also define the following expectation of $g_{j,K,t}$ over the set of random variables $\xi_{j,K,t}:=\{\xi_{j,K,t}^{(k)}\}_{k\in \mathcal{S}_{\sigma_{j}(K+1)}}$.
\begin{align}
\overline{g}_{j,K,t}=&\mathbb{E}_{\xi_{j,K,t}} [g_{j,K,t}]\nonumber\\
=&\sum_{k\in \mathcal{S}_{\sigma_{j}(K+1)}} \frac{p_k}{q_{\sigma_{j}(K+1)}} \nabla f(w_{j,K,t}^{(k)}; \mathcal{D}^{(k)}).\nonumber
\end{align}

\begin{lemma}\label{lemma_SGvar_strongconvex}
	Let \Cref{assum: P} hold. Then, $g_{j,K,t}$ and $\overline{g}_{j,K,t}$ satisfy
	\begin{align}\label{SGvar_strongconvex}
	\mathbb{E} \|\overline{g}_{j,K,t}-g_{j,K,t}\|^2 \le \sum_{K=1}^{M} q_K^{-1} \sum_{k\in \mathcal{S}_K} p_k^2 s_k^2
	\end{align}
\end{lemma}
\begin{proof}
	\begin{align}
	&\mathbb{E}\|\overline{g}_{j,K,t}-g_{j,K,t}\|^2 \nonumber\\
	&= \mathbb{E} \Big\| \sum_{k\in \mathcal{S}_{\sigma_{j}(K+1)}} \frac{p_k}{q_{\sigma_{j}(K+1)}}\nonumber\\ &\quad[\nabla f(w_{j,K,t}^{(k)}; \xi_{j,K,t}^{(k)}) - \nabla f(w_{j,K,t}^{(k)}; \mathcal{D}^{(k)})] \Big\|^2 \label{SGvar_strongconvex2}\\
	&\stackrel{(i)}{=} \mathbb{E}\Big[ \sum_{k\in \mathcal{S}_{\sigma_{j}(K+1)}} \frac{p_k^2}{q_{\sigma_{j}(K+1)}^2}\nonumber\\ &\quad\mathbb{E}_{\xi_{j, K, t}} \Big\| \nabla f(w_{j,K,t}^{(k)}; \xi_{j,K,t}^{(k)}) - \nabla f(w_{j,K,t}^{(k)}; \mathcal{D}^{(k)}) \Big\|^2 \Big]\nonumber\\
	&\stackrel{(ii)}{\le} \mathbb{E}\Big( \sum_{k\in \mathcal{S}_{\sigma_{j}(K+1)}} \frac{p_k^2 s_k^2}{q_{\sigma_{j}(K+1)}^2} \Big)\nonumber\\
	&=\sum_{K=1}^{M} q_K \sum_{k\in \mathcal{S}_K} \frac{p_k^2 s_k^2}{q_K^2} \nonumber\\
	&=\sum_{K=1}^{M} q_K^{-1} \sum_{k\in \mathcal{S}_K} p_k^2 s_k^2, \nonumber
	\end{align}
	\noindent where (i) uses the facts that $\mathbb{E}_{\xi_{j,K,t}^{(k)}} [\nabla f(w_{j,K,t}^{(k)}; \xi_{j,K,t}^{(k)}) - \nabla f(w_{j,K,t}^{(k)}; \mathcal{D}^{(k)})]=0$ and that $\sigma_{j}(K+1)$ and all the samples in $\{\xi_{j,K,t}^{(k)}\}_{k\in \mathcal{S}_{\sigma_{j}(K+1)}}$ are independent, and (ii) follows from item 2 of \Cref{assum: P}. 
\end{proof}

\begin{lemma}\label{lemma_FedAvg_howfar_strongconvex}
	Suppose the learning rate $\eta_{j,K,t}$ is non-increasing with regard to $(jM+K)E+t$. Then, for any $j,K,t$ it holds that 
	\begin{align}\label{FedAvg_howfar_sc}
		&\mathbb{E}\|\overline{w}_{j,K,t}-W_{jM}\|^2\nonumber\\ 
		&\le \mathbb{E} \Big[ \sum_{k\in \mathcal{S}_{\sigma_{j}(K+1)}} \frac{p_k}{q_{\sigma_{j}(K+1)}} \Big\|w_{j,K,t}^{(k)} - W_{jM}\Big\|^2 \Big]\nonumber\\ 
		&\le \eta_{j,0,0}^2E^2M^2G^2 
	\end{align}
\end{lemma}

\begin{proof}
	The first inequality in \eqref{FedAvg_howfar_sc} follows from Jensen's inequality applied to $\|\cdot\|^2$ and the fact that $\overline{w}_{j,K,t}=\sum_{k\in \mathcal{S}_{\sigma_{j}(K+1)}} \frac{p_k}{q_{\sigma_{j}(K+1)}} w_{j,K,t}^{(k)}$. Next, we prove the second inequality. 
	
	Define $\xi_j:=\{\xi_{j, K, t}^{(k)}: K=0,1,\ldots,M-1; t=0,1,\ldots,E-1; k\in\mathcal{S}_{\sigma_{j}(K+1)}\}$. Then, 
	\begin{align}
	&\mathbb{E}_{\xi_j}\|w_{j,K,t}^{(k)}-W_{jM}\|^2\nonumber\\
	&\le \mathbb{E}_{\xi_j}\Big\|w_{j,K,t}^{(k)}-W_{jM+K}+\sum_{K'=0}^{K-1} (W_{jM+K'+1}-W_{jM+K'})\Big\|^2\nonumber\\
	&\stackrel{(i)}{\le} (K+1)\mathbb{E}_{\xi_j}{\|w_{j,K,t}^{(k)}-w_{j,K,0}^{(k)}\|}^2\nonumber\\
	&\quad +(K+1)\sum_{K'=0}^{K-1}\mathbb{E}_{\xi_j}{\|\overline{w}_{j,K',E}-\overline{w}_{j,K',0}\|}^2 \nonumber\\
	&\stackrel{(ii)}{\le} M\mathbb{E}_{\xi_j}\Big\|\sum_{s=0}^{t-1} \eta_{j,K,s}\nabla f(w_{j,K,s}; \xi_{j,K,s}^{(k)})\Big\|^2\nonumber\\
	&\quad +M\sum_{K'=0}^{K-1}\mathbb{E}_{\xi_j}\Big\|\sum_{s=0}^{E-1} \eta_{j,K',s}g_{j,K',s}\Big\|^2 \nonumber\\
	&\stackrel{(iii)}{\le} M\Big(\sum_{s=0}^{t-1} \eta_{j,K,s}\Big) \sum_{s=0}^{t-1} \eta_{j,K,s}\mathbb{E}_{\xi_j}{\|\nabla f(w_{j,K,s}; \xi_{j,K,s}^{(k)})\|}^2\nonumber\\
	&~~~~+M \sum_{K'=0}^{K-1} \Big(\sum_{s'=0}^{E-1} \eta_{j,K',s'}\Big)\sum_{s=0}^{E-1} \eta_{j,K',s}\mathbb{E}_{\xi_j}{\|g_{j,K',s}\|}^2 \nonumber\\
	&\stackrel{(iv)}{\le} \eta_{j,0,0}^2G^2 \Big(Mt^2+MKE^2\Big) \nonumber\\
	&\stackrel{(v)}{\le} \eta_{j,0,0}^2E^2M^2G^2,\nonumber
	\end{align}
	\noindent where (i) applies Jensen's inequality to $\|\cdot\|^2$ and uses the fact that $W_{jM+K'+1}=\overline{w}_{j,K',E}=\overline{w}_{j,K'+1,0}=w_{j,K'+1,0}^{(k)}$ ($0\le K'\le M-2$, $k\in\mathcal{S}_{\sigma_{j}(K'+1)}$) based on the update rule of \Cref{alg: 1}, (ii) uses \eqref{w_avg_update} and $K\le M-1$, (iii) applies Jensen's inequality, (iv) uses item 3 of \Cref{assum: P}, the inequality \eqref{g_bound} below and the fact that $\eta_{j,K,s}\le\eta_{j,0,0}$ for $K, s\ge 0$, and (v) uses the fact that $K\le M-1$ and $0\le t\le E-1$. 
	\begin{align}\label{g_bound}
	&\mathbb{E}_{\xi_j}\|g_{j,K,t}\|^2\nonumber\\
	&=\mathbb{E}_{\xi_j}\Big\|\sum_{k\in \mathcal{S}_{\sigma_{j}(K+1)}} \frac{p_k}{q_{\sigma_{j}(K+1)}} \nabla f(w_{j,K,t}^{(k)}; \xi_{j,K,t}^{(k)})\Big\|^2\nonumber\\
	&\le\mathbb{E}_{\xi_j}\Big[\sum_{k\in \mathcal{S}_{\sigma_{j}(K+1)}} \frac{p_k}{q_{\sigma_{j}(K+1)}} \Big\|\nabla f(w_{j,K,t}^{(k)}; \xi_{j,K,t}^{(k)})\Big\|^2\Big]\nonumber\\
	&\le \mathbb{E}_{\xi_j}\Big[\sum_{k\in \mathcal{S}_{\sigma_{j}(K+1)}} \frac{p_k}{q_{\sigma_{j}(K+1)}} G^2\Big]\le G^2.
	\end{align}
\end{proof}

\section{Proof of \Cref{thm: nc}}
By the $L$-smoothness of the objective function $f$, we obtain that
\begin{align}\label{err1step_nc}
&\mathbb{E}[f(\overline{w}_{j,K,t+1})-f(\overline{w}_{j,K,t})]\nonumber\\
&\le \mathbb{E}\langle \nabla f(\overline{w}_{j,K,t}), \overline{w}_{j,K,t+1} - \overline{w}_{j,K,t}\rangle \nonumber\\
&\quad+ \frac{L}{2}\mathbb{E}\|\overline{w}_{j,K,t+1} - \overline{w}_{j,K,t}\|^2\nonumber\\
&\stackrel{(i)}{=} \mathbb{E}\Big\langle \nabla f(\overline{w}_{j,K,t}), \nonumber\\
&\quad-\eta_{j,K,t}\sum_{k\in \mathcal{S}_{\sigma_{j}(K+1)}} \frac{p_k}{q_{\sigma_{j}(K+1)}} \nabla f(w_{j,K,t}^{(k)}; \xi_{j,K,t}^{(k)}) \Big\rangle \nonumber\\
&\quad+\frac{L}{2} \mathbb{E}\Big\| \eta_{j,K,t} \sum_{k\in \mathcal{S}_{\sigma_{j}(K+1)}} \frac{p_k}{q_{\sigma_{j}(K+1)}} \nabla f(w_{j,K,t}^{(k)}; \xi_{j,K,t}^{(k)}) \Big\|^2\nonumber\\
&\stackrel{(ii)}{\le} -\eta_{j,K,t} \mathbb{E}\Big\langle \nabla f(\overline{w}_{j,K,t}), \nonumber\\
&\quad\sum_{k\in \mathcal{S}_{\sigma_{j}(K+1)}} \frac{p_k}{q_{\sigma_{j}(K+1)}} \nabla f(w_{j,K,t}^{(k)}; \mathcal{D}^{(k)}) \Big\rangle \nonumber\\
&\quad+ 2L\eta_{j,K,t}^2 \mathbb{E}\Big\| \nabla f(W_{jM}) \Big\|^2\nonumber\\
&\quad+2L\eta_{j,K,t}^2 \mathbb{E}\Big\| \sum_{k\in \mathcal{S}_{\sigma_{j}(K+1)}} \frac{p_k}{q_{\sigma_{j}(K+1)}} \nabla f(W_{jM}; \mathcal{D}^{(k)}) \nonumber\\
&\quad- \nabla f(W_{jM}) \Big\|^2 +2L\eta_{j,K,t}^2 \mathbb{E}\Big\| \sum_{k\in \mathcal{S}_{\sigma_{j}(K+1)}} \frac{p_k}{q_{\sigma_{j}(K+1)}} \nonumber\\
&\quad\quad[\nabla f(w_{j,K,t}^{(k)}; \mathcal{D}^{(k)})-\nabla f(W_{jM}; \mathcal{D}^{(k)})] \Big\|^2\nonumber\\
&\quad+2L\eta_{j,K,t}^2 \mathbb{E}\Big\| \sum_{k\in \mathcal{S}_{\sigma_{j}(K+1)}} \frac{p_k}{q_{\sigma_{j}(K+1)}} \nonumber\\
&\quad\quad[\nabla f(w_{j,K,t}^{(k)}; \xi_{j,K,t}^{(k)})-\nabla f(w_{j,K,t}^{(k)}; \mathcal{D}^{(k)})] \Big\|^2\nonumber\\
&\stackrel{(iii)}{\le} -\eta_{j,K,t} \mathbb{E}\Big\langle \nabla f(\overline{w}_{j,K,t}), \nonumber\\
&\quad\sum_{k\in \mathcal{S}_{\sigma_{j}(K+1)}} \frac{p_k}{q_{\sigma_{j}(K+1)}} [\nabla f(w_{j,K,t}^{(k)}; \mathcal{D}^{(k)}) - \nabla f(W_{jM}; \mathcal{D}^{(k)})] \Big\rangle \nonumber\\
&\quad-\eta_{j,K,t} \mathbb{E}\Big\langle \nabla f(\overline{w}_{j,K,t}) - \nabla f(W_{jM}),\nonumber\\
&\quad\sum_{k\in \mathcal{S}_{\sigma_{j}(K+1)}} \frac{p_k}{q_{\sigma_{j}(K+1)}} \nabla f(W_{jM}; \mathcal{D}^{(k)}) \Big\rangle \nonumber\\
&\quad-\eta_{j,K,t} \mathbb{E}\Big\langle \nabla f(W_{jM}), \sum_{k\in \mathcal{S}_{\sigma_{j}(K+1)}} \frac{p_k}{q_{\sigma_{j}(K+1)}} \nabla f(W_{jM}; \mathcal{D}^{(k)}) \Big\rangle \nonumber\\
&\quad+2L\eta_{j,K,t}^2 \mathbb{E}\Big\| \nabla f(W_{jM}) \Big\|^2\nonumber\\
&\quad+2L\eta_{j,K,t}^2 \mathbb{E}\Big\| \nabla f(W_{jM}; \mathcal{D}^{(\mathcal{S}_{\sigma_{j}(K+1)})}) - \nabla f(W_{jM}) \Big\|^2\nonumber\\
&\quad+2L\eta_{j,K,t}^2 \mathbb{E} \Big[ \sum_{k\in \mathcal{S}_{\sigma_{j}(K+1)}} \frac{p_k}{q_{\sigma_{j}(K+1)}} \nonumber\\
&\quad\quad\Big\|\nabla f(w_{j,K,t}^{(k)}; \mathcal{D}^{(k)})-\nabla f(W_{jM}; \mathcal{D}^{(k)}) \Big\|^2 \Big]\nonumber\\
&\quad+2L\eta_{j,K,t}^2 \sum_{K=1}^{M} q_K^{-1} \sum_{k\in \mathcal{S}_K} p_k^2 s_k^2\nonumber\\
&\stackrel{(iv)}{\le} \eta_{j,K,t} \mathbb{E}\Big\{\Big\| \nabla f(\overline{w}_{j,K,t})\Big\|\Big\|\sum_{k\in \mathcal{S}_{\sigma_{j}(K+1)}} \frac{p_k}{q_{\sigma_{j}(K+1)}} \nonumber\\
&\quad [\nabla f(w_{j,K,t}^{(k)}; \mathcal{D}^{(k)}) - \nabla f(W_{jM}; \mathcal{D}^{(k)})] \Big\| \Big\} \nonumber\\
&\quad+\eta_{j,K,t} \mathbb{E}\Big\{\Big\|\nabla f(\overline{w}_{j,K,t}) - \nabla f(W_{jM})\Big\| \nonumber\\
&\quad\Big\|\nabla f(W_{jM}; \mathcal{D}^{(\mathcal{S}_{\sigma_{j}(K+1)})})\Big\|\Big\}\nonumber\\
&\quad-\eta_{j,K,t} \mathbb{E}\Big\| \nabla f(W_{jM}) \Big\|^2 + 2L\eta_{j,K,t}^2 \mathbb{E}\Big\| \nabla f(W_{jM}) \Big\|^2\nonumber\\
&\quad+2L\eta_{j,K,t}^2 H_{\text{cluster}}\nonumber\\
&\quad + 2L^3\eta_{j,K,t}^2 \mathbb{E} \Big[ \sum_{k\in \mathcal{S}_{\sigma_{j}(K+1)}} \frac{p_k}{q_{\sigma_{j}(K+1)}} \Big\|w_{j,K,t}^{(k)} - W_{jM}\Big\|^2 \Big]\nonumber\\
&\quad+2L\eta_{j,K,t}^2 \sum_{K=1}^{M} q_K^{-1} \sum_{k\in \mathcal{S}_K} p_k^2 s_k^2\nonumber\\
&\stackrel{(v)}{\le} \eta_{j,K,t}G \mathbb{E} \Big\{ \sum_{k\in \mathcal{S}_{\sigma_{j}(K+1)}} \frac{p_k}{q_{\sigma_{j}(K+1)}} \nonumber\\
&\quad\Big\|\nabla f(w_{j,K,t}^{(k)}; \mathcal{D}^{(k)}) - \nabla f(W_{jM}; \mathcal{D}^{(k)}) \Big\| \Big\} \nonumber\\
&\quad+\eta_{j,K,t}LG \mathbb{E} \Big\|\overline{w}_{j,K,t} - W_{jM}\Big\|\nonumber\\ &\quad-\eta_{j,K,t}(1-2L\eta_{j,K,t}) \mathbb{E}\Big\| \nabla f(W_{jM}) \Big\|^2\nonumber\\
&\quad+2L\eta_{j,K,t}^2 H_{\text{cluster}} + 2L^3\eta_{j,K,t}^2 \eta_{j,0,0}^2E^2M^2G^2\nonumber\\
&\quad+2L\eta_{j,K,t}^2 \sum_{K=1}^M q_K^{-1} \sum_{k\in \mathcal{S}_K} p_k^2 s_k^2\nonumber\\
&\stackrel{(vi)}{\le} \eta_{j,K,t}LG \sqrt{\mathbb{E} \Big[ \sum_{k\in \mathcal{S}_{\sigma_{j}(K+1)}} \frac{p_k}{q_{\sigma_{j}(K+1)}} \Big\|w_{j,K,t}^{(k)} - W_{jM}\Big\|^2 \Big]} \nonumber\\
&\quad+\eta_{j,K,t}LG \sqrt{\mathbb{E} \Big\|\overline{w}_{j,K,t} - W_{jM}\Big\|^2} -\frac{\eta_{j,K,t}}{2} \mathbb{E}\Big\| \nabla f(W_{jM}) \Big\|^2\nonumber\\
&\quad+2L\eta_{j,K,t}^2 H_{\text{cluster}} + 2L^3\eta_{j,K,t}^2 \eta_{j,0,0}^2E^2M^2G^2\nonumber\\
&\quad+2L\eta_{j,K,t}^2 \sum_{K=1}^M q_K^{-1} \sum_{k\in \mathcal{S}_K} p_k^2 s_k^2\nonumber\\
&\stackrel{(vii)}{\le} 2\eta_{j,K,t}LG \sqrt{\eta_{j,0,0}^2 E^2M^2G^2} -\frac{\eta_{j,K,t}}{2} \mathbb{E}\Big\| \nabla f(W_{jM}) \Big\|^2\nonumber\\
&\quad+2L\eta_{j,K,t}^2 H_{\text{cluster}} + 2L^3\eta_{j,K,t}^2 \eta_{j,0,0}^2E^2M^2G^2\nonumber\\
&\quad+2L\eta_{j,K,t}^2 \sum_{K=1}^M q_K^{-1} \sum_{k\in \mathcal{S}_K} p_k^2 s_k^2\nonumber\\
&\stackrel{(viii)}{=} \Big(\frac{2L}{T}+\frac{2L^3}{T^2}\Big) G^2 -\frac{1}{2\sqrt{TME}} \mathbb{E}\Big\| \nabla f(W_{jM}) \Big\|^2\nonumber\\
&\quad+\frac{2L}{TME} \Big(H_{\text{cluster}}+\sum_{K=1}^M q_K^{-1} \sum_{k\in \mathcal{S}_K} p_k^2 s_k^2\Big),
\end{align}
\noindent where (i) uses \eqref{w_avg_update} \& \eqref{g}, (ii) uses the equality that $\mathbb{E}_{\xi_{j,K,t}} \nabla f(w_{j,K,t}^{(k)};\xi_{j,K,t}^{(k)})=\nabla f(w_{j,K,t}^{(k)}; \mathcal{D}^{(k)})$ and the inequality that $\|\sum_{k=1}^{m} x_k\|^2\le m\sum_{k=1}^{m} \|x_k\|^2$, (iii) applies Jensen's inequality to $\|\cdot\|^2$ and uses \eqref{SGvar_strongconvex} \& \eqref{SGvar_strongconvex2}, (iv) uses Cauchy-Schwartz inequality, the heterogeneity definition in \eqref{H_cluster}, the $L$-smoothness of $f(\cdot; \mathcal{D}^{(k)})$, the equality that $\mathbb{E}_{\sigma_{j}(K+1)} \Big[\sum_{k\in \mathcal{S}_{\sigma_{j}(K+1)}} \frac{p_k}{q_{\sigma_{j}(K+1)}} \nabla f(W_{jM}; \mathcal{D}^{(k)})\Big]=\nabla f(W_{jM})$ and the fact that $W_{jM}$ is independent of $\sigma_{j}(K+1)$, {(v) uses \ref{df_G} \& \ref{FedAvg_howfar_sc}}, the $L$-smoothness of $f$ and $W_{jM}=\overline{w}_{j,0,0}$ and applies Jensen's inequality to $\|\cdot\|^2$, (vi) uses the $L$-smoothness of $f(\cdot; \mathcal{D}^{(k)})$, $\eta_{j,K,t}\equiv(TME)^{-\frac{1}{2}}\le\frac{1}{4L}$ (since $T\ge \frac{16L^2}{EM}$) and the inequality that $\mathbb{E}\|X\|\le\sqrt{\mathbb{E}\|X\|^2}$ for any random vector $X$, (vii) uses \eqref{FedAvg_howfar_sc} and (viii) substitutes $\eta_{j,K,t}\equiv(TME)^{-\frac{1}{2}}$ into this equation. 

Note that $W_{jM+K'+1}=\overline{w}_{j,K',E}=\overline{w}_{j,K'+1,0}$ ($0\le K'\le M-2$), and $W_{(j+1)M}=\overline{w}_{j,M-1,E}=\overline{w}_{j+1,0,0}$. Hence, by telescoping \eqref{err1step_nc} over $j=0,1,\ldots,T-1$, $K=0,1,\ldots,M-1$, $t=0,1,\ldots,E-1$, it holds that
\begin{align}
&\mathbb{E}[f(W_{TM})-f(W_{0})]\nonumber\\
&\le 2LMEG^2\Big(1+\frac{L^2}{T}\Big)  -\frac{1}{2}\sqrt{\frac{ME}{T}} \sum_{j=0}^{T-1} \mathbb{E}\Big\| \nabla f(W_{jM}) \Big\|^2\nonumber\\
&\quad +2L\Big(H_{\text{cluster}}+\sum_{K=1}^M q_K^{-1} \sum_{k\in \mathcal{S}_K} p_k^2 s_k^2\Big),\nonumber
\end{align}
which by using $f(W_{TM})\ge \inf_{w} f(w)$ and the assumption that $T\ge L^2$ further implies 
\begin{align}\label{err_nc_general}
\frac{1}{T}\sum_{j=0}^{T-1} \mathbb{E}\Big\| \nabla f(W_{jM}) \Big\|^2 \le& \frac{C}{\sqrt{TME}} +C_2\sqrt{\frac{ME}{T}} \le \frac{2C}{\sqrt{TME}},
\end{align}
where $C$ is defined in \eqref{C}, $C_2=8LG^2$, and the second inequality uses the assumption that $ME\le C/C_2=\frac{C}{8LG^2}$. Therefore, to achieve an $\epsilon$-stationary point, we need
\begin{align}
\frac{2C}{\sqrt{TME}}\le\epsilon,
\end{align}
which implies that
\begin{align}
T\ge \frac{4C^2\epsilon^{-2}}{ME} \propto \frac{C^2}{ME}.
\end{align}